\documentclass[runningheads]{llncs}
\usepackage{graphicx}
\usepackage{amsmath,amssymb} % define this before the line numbering.
\usepackage{color}
\usepackage{epstopdf}
\usepackage{epsfig}
\usepackage[width=122mm,left=12mm,paperwidth=146mm,height=193mm,top=12mm,paperheight=217mm]{geometry}
\usepackage{stackengine}
\begin{document}
\title{Fast and Accurate Intrinsic Symmetry Detection} 
% Replace with your title

\titlerunning{Fast and Accurate Intrinsic Symmetry Detection}
% Replace with a meaningful short version of your title
%
\author{Rajendra Nagar \and Shanmuganathan Raman}
%
%Please write out author names in full in the paper, i.e. full given and family names. 
%If any authors have names that can be parsed into FirstName LastName in multiple ways, please include the correct parsing, in a comment to the volume editors:
%\index{Lastnames, Firstnames}
%(Do not uncomment it, because you may introduce extra index items if you do that, we will use scripts for introducing index entries...)
\authorrunning{R. Nagar and S. Raman}
% Replace with shorter version of the author list. If there are more authors than fits a line, please use A. Author et al.
%

\institute{Electrical Engineering, Indian Institute of Technology Gandhinagar, India\\
\email{ \{rajendra.nagar,shanmuga\}@iitgn.ac.in}}
\maketitle              % typeset the header of the contribution
\begin{abstract}
	In computer vision and graphics, various types of symmetries are extensively studied since symmetry present in objects is a fundamental cue for understanding the shape and the structure of objects. In this work, we detect the intrinsic reflective symmetry in triangle meshes where we have to find the intrinsically symmetric point for each point of the shape. We establish correspondences between functions defined on the shapes by extending the functional map framework and then recover the point-to-point correspondences. Previous approaches using the functional map for this task find the functional correspondences matrix by solving a non-linear optimization problem which makes them slow. In this work, we propose a closed form solution for this matrix which makes our approach faster. We find the closed-form solution based on our following results.  If the given shape is intrinsically symmetric, then the shortest length geodesic between two intrinsically symmetric points is also intrinsically symmetric. If an eigenfunction of the Laplace-Beltrami operator for the given shape is an even (odd) function, then its restriction on the shortest length geodesic between two intrinsically symmetric points is also an even (odd) function. The sign of a low-frequency eigenfunction is the same on the neighboring points. Our method is invariant to the ordering of the eigenfunctions and has the least time complexity. We achieve the best performance on the SCAPE dataset and comparable performance with the state-of-the-art methods on the TOSCA dataset.
\end{abstract}
	\section{Introduction}
	The importance of various types of symmetry is evident while solving problems such as shape segmentation, mesh repairing, shape matching, retrieving the normal forms of 3D models \cite{mitra2007symmetrization,zheng2015skeleton,ghosh2010closed}, inverse procedural modeling \cite{mitra2008symmetry}, shape recognition \cite{kazhdan2004symmetry}, shape understanding \cite{mitra2010illustrating}, shape completion \cite{speciale2016symmetry,sung2015data}, and shape editing \cite{xiao2015content,mitra2014structure,kerber2013scalable,kurz2014symmetry,wu2014real,dessein2017symmetry}. The problem of detecting intrinsic symmetry is shown to be an NP-hard problem since it amounts to finding an intrinsically symmetric point for each point \cite{ovsjanikov2008global}. However, correspondences between the intrinsically symmetric points completely characterize the intrinsic symmetry of a shape since the intrinsic symmetry is a non-rigid transformation which can not be represented a matrix as opposed to the extrinsic symmetry which can be represented by a matrix \cite{kim2010mobius}.  We exploit the functional map approach that finds correspondences between functions instead of points \cite{ovsjanikov2012functional}. Then, the point-to-point correspondences can be recovered in $O(n\log (n))$, where $n$ is the number of vertices. We extend this framework for the detection of intrinsic symmetry. The functional map framework has already been used for detecting intrinsic symmetry in previous works  (\cite{liu2015properly}, \cite{wang2017group}). The main task in these approaches was to determine the functional correspondence matrix which transforms a function to its intrinsic image. Various constraints have been enforced on this matrix. It is not known yet, how many constraints are sufficient. Further, they solved a non-linear optimization problem to estimate the functional correspondence matrix which makes the method slow. For the intrinsic symmetry detection problem, we show that the functional correspondences matrix is a diagonal matrix and a diagonal entry is $+1$ ($-1$), if the corresponding eigenfunction is an even (odd) function. This closed-form solution makes our method faster.
	
	We determine if a particular eigenfunction is even or odd based on our following result. An eigenfunction is an even (odd) function if its restriction to the shortest length geodesic between two intrinsically symmetric points is an even (odd) function. Therefore, we need to find a few accurate pairs of intrinsically symmetric points which we find using our following results. If we directly pair points based on the similarity between their heat kernel signatures \cite{sun2009concise}, we may get false pairs. For example, a pair of points on the tip of the index finger and the tip of the ring finger of the same hand of a human model. The reason is that, if two neighboring points are subjected to the same strength heat sources, then their heat diffusion processes will also be similar because of the very small sizes of the fingers with respect to the body size. However, we observer that the sign of a low-frequency eigenfunction on the neighboring points is the same. Hence, we put high penalty for pairing two points if signs of first few low-frequency eigenfunctions are the same for both the points. The models of the benchmark datasets are obtained by applying an imperfect isometry, so the theory only holds approximately. Furthermore, some of the triangles may not be Delaunay triangles.  Therefore, we may not get accurate correspondences using the original eigenfunctions. Hence, we transform the original eigenfunctions to make them near perfect even or odd functions. Following are our main contributions.
	\begin{enumerate}
		\item We propose a novel approach to find a few accurate pairs of intrinsically symmetric points based on the following property of eigenfunctions: the signs of low-frequency eigenfunction on neighboring points are the same.
		\item We propose a novel and efficient approach for finding the functional correspondence matrix. We prove that the functional matrix for the intrinsic symmetry detection problem is a diagonal matrix and a diagonal entry is $+1\;(-1)$ if the corresponding eigenfunction is an even (odd) function.
		\item We propose a novel approach to determine the sign of an eigenfunction by showing that,  if a manifold contains intrinsic symmetry and an eigenfunction is an even (odd) function, then its restriction to the shortest length geodesic between any two intrinsically symmetric points is an even (odd) function.
		\item We transform the eigenfunctions to make them more invariant to self-isometry. 
	\end{enumerate}    
	\section{Related Works}
	Reflective symmetry detection methods are categorized by the types of the input data and the reflective symmetry present in the input data. The input data can be a digital image, point cloud, triangle mesh, etc. The main types of reflective symmetries are extrinsic and intrinsic. The well known methods for detecting the extrinsic symmetry are \cite{mitra2006partial,podolak2006planar,martinet2006accurate,li2016efficient,shi2016symmetry,thomas2013detecting,sipiran2014approximate,speciale2016symmetry,loy2006detecting} and for detecting the intrinsic symmetry are \cite{raviv2010full,ovsjanikov2008global,berner2009generalized,berner2011shape,jiang2013skeleton,kim2011blended,korman2014probably,li2015approximate,lukavc2017nautilus,raviv2010full,shi2016symmetry,zheng2015skeleton,yoshiyasu2016symmetry,xu2012multi,wang2011symmetry,shehu2014characterization,panozzo2012fields,raviv2010diffusion}. Furthermore, the intrinsic symmetry can be characterized as global and partial intrinsic symmetry. Our method finds global and partial intrinsic symmetry in triangle meshes. We discuss only the relevant works and suggest the readers to follow the excellent state-of-the-art report in \cite{mitra2013symmetry} and the survey in \cite{liu2010computational}.
	
	Ovsjanikov \emph{et al.} detected intrinsic symmetry using the global point signature (GPS) \cite{ovsjanikov2008global}. The main claim was that the GPS embedding transforms the problem from intrinsic to extrinsic symmetry detection. They showed that the GPS embedding was robust to the topological noises. They found pairs of symmetric points by comparing the GPS of one point to the signed-GPS of the other point. The time complexity of determining the sign (even or odd) of an eigenfunction is $O(n\log (n))$ since they compared GPSs of all the possible pairs. Furthermore, the time complexity of the overall method is $O(k^3 n\log (n))$ excluding the computation of eigenfunctions, where $k$ is the number of eigenfunctions used. We propose a more efficient method which takes ($O(kn\log(n))$) for detecting intrinsic symmetry. Furthermore, we observe that the approach by \cite{ovsjanikov2008global} is sensitive to the sign flip and eigenfunction ordering. Whereas, our method is independent of sign flip and ordering since we determine the sign of each eigenfunction independently from the others. Mitra \emph{et al.} used a voting based approach to detect intrinsic symmetry and then applied transformation in the voting space to deform the input model to have perfect extrinsic symmetry \cite{mitra2007symmetrization}.  Xu \emph{et al.} used a generalized voting scheme to find the partial intrinsic symmetry curve without explicitly finding the intrinsically symmetric point for each point \cite{xu2009partial}. Xu \emph{et al.} efficiently found pairs of intrinsically symmetric points using a voting based approach \cite{xu2012multi}. They factored out symmetry based on the scale of symmetry. However, they needed to tune a parameter depending on how much the input shape is distorted \cite{xu2012multi}. The methods by Zheng \emph{et al.} (\cite{zheng2015skeleton}, \cite{jiang2013skeleton}) also used voting approach. These voting based methods do not utilize spatial coherency. Therefore, they may produce pairs of intrinsically symmetric points which may not be spatially continuous.  Furthermore, they may have high complexity due to a large number of possible pairs for the voting. 
	
	In \cite{raviv2010full}, the authors proposed a non-convex optimization framework to accurately detect full and partial symmetries of 3D models. However, the initialization severely affects the performance, and the complexity also is very high. Lipman \emph{et al.} efficiently found the pairs of intrinsically symmetric points in point clouds and triangle meshes using the novel symmetry factored embedding technique \cite{lipman2010symmetry}. However, their main bottleneck is the time complexity which is $O(n^{2.5}\log(n))$. Kim \emph{et al.} used anti-M{\"o}bius transformation to accurately find intrinsically symmetric pairs of points \cite{kim2010mobius}. They first find a sparse set of pairs of intrinsically symmetric points. Then, they transform intrinsic symmetry into extrinsic symmetry using the M{\"o}bius transform.  However, they required $O(n^2)$ space for mid-edge flattening and $O(|\mathcal{S}|^4)$ for finding the anti-M{\"o}bius transformation, where $\mathcal{S}$ is the set of symmetry invariant points. Furthermore, false pairs in the first step may severely affect the overall performance. 
	\section{Intrinsic Symmetry Detection}
	\subsection{Background}
	Let $\mathcal{M}$ be a compact and connected 2-manifold representing the input shape. Let $L^2(\mathcal{M})=\{f:\mathcal{M}\rightarrow\mathbb{R}|\left\langle f,f\right\rangle_\mathcal{M}=\int_{\mathcal{M}}f^2(x)dx<\infty\}$ be the space of square integrable functions defined on $\mathcal{M}$. The Laplace-Beltrami operator on a shape $\mathcal{M}$ is defined as $\Delta_\mathcal{M}f=-\text{div}_\mathcal{M}(\nabla_\mathcal{M}f)$ and admits an eigenvalue decomposition $
	\Delta_\mathcal{M}\phi_i(x)=\lambda_i\phi_i(x),\forall x\in\mathcal{M}$. Here, $0=\lambda_1\leq\lambda_2\leq\ldots$ are the eigenvalues and $\phi_1,\phi_2,\ldots$ are the corresponding eigenfunctions. The eigenfunctions $\phi_1,\phi_2,\ldots$ form a basis for the space $L^2(\mathcal{M})$. Therefore, any function $f\in L^2(\mathcal{M})$ can be represented as $
	f(x)=\sum_{i=1}^{\infty}\left\langle f,\phi_i\right\rangle_\mathcal{M}\phi_i(x), \forall x\in\mathcal{M}.
	$ 
	The functional map framework was first proposed in \cite{ovsjanikov2012functional} for establishing point-to-point dense correspondence between two isometric shapes. The main idea was to establish correspondences between the functions, defined on the shapes, rather than the points. This idea reduced the time complexity to $O(n\log n)$. Let $\mathcal{M}$ and $\mathcal{N}$ be two shapes. Let $T_{\text{f}}:L^2(\mathcal{N})\rightarrow L^2(\mathcal{M})$ be a linear mapping between functions defined on these shapes.  That is, if $g:\mathcal{N}\rightarrow \mathbb{R}$ and $f:\mathcal{M}\rightarrow \mathbb{R}$ are two corresponding functions then $T_{\text{f}}(g)=f$. The mapping $T_{\text{f}}$ is represented by a matrix $\mathbf{C}\in\mathbb{R}^{k\times k}$ such that $\mathbf{b}=\mathbf{Ca}$, where $\mathbf{a}=\begin{bmatrix} a_1&a_2&\ldots&a_k\end{bmatrix}^\top$  and $\mathbf{b}=\begin{bmatrix} b_1&b_2&\ldots&b_k\end{bmatrix}^\top$ are the representations of the functions $g$ and $f$ in the truncated bases $\{\phi_i^\mathcal{N}\}_{i=1}^k$ and $\{\phi_i^\mathcal{M}\}_{i=1}^k$, respectively. Therefore, the main goal is to find the matrix $\mathbf{C}$ which completely characterizes the dense correspondence between the two shapes.
	\subsection{Functional Maps for Intrinsic Symmetry Detection}
	We extend the functional map framework for detecting the intrinsic symmetry which can be thought of as a shape correspondence problem where we have to find the correspondences between the points of the same shape rather than the points on the two different shapes. The functional map framework is applicable for two isometric shapes also. Therefore, we can use it to detect the intrinsic symmetry since a symmetric shape is a self-isometric shape \cite{ovsjanikov2008global}. The intrinsic symmetry $T_{\text{p}}:\mathcal{M}\rightarrow\mathcal{M}$ of $\mathcal{M}$ is defined as follows. If the points $x\in\mathcal{M}$ and $y\in\mathcal{M}$ are intrinsically symmetric, then $T_{\text{p}}(x)=y$ and $T_{\text{p}}(y)=x$. We first find the mapping between the functions and then use it to find the correspondences between the intrinsically symmetric points. Let us consider the space $L^2(\mathcal{M})$ and let $T:L^2(\mathcal{M})\rightarrow L^2(\mathcal{M})$ be a functional map which maps the functions defined on the same shape. Then, this functional map  $T$ completely characterizes the intrinsic symmetry $T_{\text{p}}$ if $T(g)=f$ and $T(f)=g$, where $f,g\in L^2(\mathcal{M})$ are intrinsically symmetric functions, i.e. $f\circ T_\text{p}(x)=g(x),\;g\circ T_\text{p}(x)=f(x), \forall x\in\mathcal{M}$. Therefore, our goal is to find the matrix $\mathbf{C}$ which characterizes the functional mapping $T$ for the intrinsic symmetry detection problem. For the problem of finding correspondences between two shapes, various constraints have been imposed on the matrix $\mathbf{C}$. Then, the matrix $\mathbf{C}$ was the optimal solution of an optimization problem. However, we show that a closed form solution exists for the matrix $\mathbf{C}$ for the problem of detecting the intrinsic symmetry, which we state as follows.
	\begin{theorem}
		\label{th_1}
		Let $T:L^2(\mathcal{M})\rightarrow L^2(\mathcal{M})$ be a mapping between the functions defined on a shape $\mathcal{M}$ and $T$ characterizes the intrinsic symmetry $T_\text{p}$ of $\mathcal{M}$, i.e. $T(g)=f, T(f)=g,\;\forall f,g\in L^2(\mathcal{M})$ such that $f\circ T_\text{p}(x)=g(x),\;g\circ T_\text{p}(x)=f(x), \forall x\in\mathcal{M}$. Then, the matrix $\mathbf{C}$ representing $T$ is a diagonal matrix. $\mathbf{C}_{i,i}=+1$, if $\langle T(\phi_i),\phi_i\rangle_\mathcal{M}=+1$, and $\mathbf{C}_{i,i}=-1$, if $\langle T(\phi_i),\phi_i\rangle_\mathcal{M}=-1$.
	\end{theorem}
	\begin{proof}
		The functions $f,g$ belong to the space $L^2(\mathcal{M})$. We can represent $f$ and $g$ as $f(x)=\sum_{i=1}^{\infty}b_i\phi_i(x)$ and $g(x)=\sum_{i=1}^{\infty}a_i\phi_i(x)$.  Since $T$ is a linear mapping, we have
		$ T(f)=T\left(\sum_{i=1}^{\infty}b_i\phi_i(x)\right)  
		= \sum_{i=1}^{\infty}b_i T\left(\phi_i(x)\right)$. Since $T\left(\phi_i(x)\right)$ is also a function in the space $L^2(\mathcal{M} )$, it can be represented in the basis $\{\phi_i\}_{i=1}^\infty$ as $T\left(\phi_i(x)\right)=\sum_{j=1}^{\infty}c_{ij}\phi_j(x)$,  where $c_{ij}=\langle T(\phi_i),\phi_j\rangle_\mathcal{M}$. Therefore, we have that $T(f) = \sum_{j=1}^{\infty} \sum_{i=1}^{\infty}c_{ij}b_i \phi_j(x)$. Since $T(f)=g$, it follows that $\sum_{j=1}^{\infty} \sum_{i=1}^{\infty}c_{ij}b_i \phi_j(x)=\sum_{j=1}^{\infty}a_j\phi_j(x)$. Therefore, $a_j=\sum_{i=1}^{\infty}c_{ij}b_i$. Equivalently, we can write it as $\mathbf{a}=\mathbf{C}\mathbf{b}$, where $\mathbf{C}_{i,j}=c_{ij}=\langle T(\phi_i),\phi_j\rangle_\mathcal{M}$. According to \cite{ovsjanikov2008global}, the eigenfunctions (corresponding to the non-repeating eigenvalues) are self-isometry invariant with sign ambiguity i.e. $\phi_i\circ T_{\text{p}}(x)=\pm\phi_i(x)$, $\forall x\in\mathcal{M}$. Furthermore, the functional map $T$ completely characterizes the intrinsic symmetry. Therefore,  $T(\phi_i)=+\phi_i$, if $\phi_i\circ T_{\text{p}}(x)=\phi_i(x)$, and $T(\phi_i)=-\phi_i$, if $\phi_i\circ T_{\text{p}}(x)=-\phi_i(x)$, $\forall x\in\mathcal{M}$.        
		Since the eigenfunctions $\phi_1, \phi_2,\ldots $ form an orthogonal basis for the space $L^2(\mathcal{M})$, we have that $\langle \pm\phi_i,\phi_j\rangle_\mathcal{M}=0$ if $i\neq j$. Therefore, $\mathbf{C}_{i,j}=\langle T(\phi_i),\phi_j\rangle_\mathcal{M}=\langle \pm\phi_i,\phi_j\rangle_\mathcal{M}$. Therefore, $\mathbf{C}_{i,j}=0$, if $i\neq j$. Hence, the matrix $\mathbf{C}$ is a diagonal matrix. Furthermore, $\mathbf{C}_{i,i}=+1$, if $\langle T(\phi_i),\phi_i\rangle_\mathcal{M}=+1$, and $\mathbf{C}_{i,i}=-1$, if $\langle T(\phi_i),\phi_i\rangle_\mathcal{M}=-1$.\qed
	\end{proof}
	Therefore, the problem of determining whether $\mathbf{C}_{i,i}=+1$ or  $\mathbf{C}_{i,i}=-1$ is equivalent to determining whether $\phi_i\circ T_\text{p}(x)=+\phi_i(x)$ or $\phi_i\circ T_\text{p}(x)=-\phi_i(x)$, $\forall x\in\mathcal{M}$. It is observed that if $\phi_i\circ T_\text{p}=+\phi_i$ then $\phi_i$ is a symmetric or an even function and if $\phi_i\circ T_\text{p}=-\phi_i$, then $\phi_i$ is an anti-symmetric or an odd function in the intrinsic sense. We can not apply the definition of the vector space here, since the domain of the eigenfunctions is not a vector space. A function $f:\mathbb{R}^2\rightarrow\mathbb{R}$ is an even function, if $f(-\mathbf{x})=f(\mathbf{x}),\forall\mathbf{x}\in\mathbb{R}^2$. This definition is not valid for the functions defined on the manifolds, since if $x\in\mathcal{M}$, then it may not always be true that $-x\in\mathcal{M}$. However, we generalize the following property of vector spaces to the manifolds to determine the sign of eigenfunctions. Let $f:\mathbb{R}^2\rightarrow\mathbb{R}$ be a function symmetric on $\mathbb{R}^2$ and $\ell=\{\mathbf{x}:\mathbf{x}=t\mathbf{x}_1+(1-t)\mathbf{x}_2, t\in[0,1]\}$ be the line segment joining the mirror symmetric points $\mathbf{x}_1$ and $\mathbf{x}_2$. Then, it is trivial to show that the restriction  $f_{\ell}:\ell\rightarrow\mathbb{R}$ of the function $f$ on the set $\ell$ is also symmetric. Here, we also observe that the set $\mathbb{R}^2$ is symmetric about any of its coordinate axes and the set $\ell$ is also symmetric. We formally generalize these results on manifolds as follows.
	\begin{theorem}
		\label{th_2}
		Let $\mathcal{M}  $ be a compact and connected 2-manifold. Let there exist a self-isometry $T_\text{p}:\mathcal{M}\rightarrow\mathcal{M}$ on $\mathcal{M}$.
		Let $x,y\in\mathcal{M}$ be two points which are intrinsically symmetric, i.e., $T_\text{p}(x)=y$ and $T_\text{p}(y)=x$. Let $\gamma(t):[0,1]\rightarrow\mathcal{M}$ be the shortest length geodesic curve between the points $x$ and $y$ such that $\gamma(0)=x$ and $\gamma(1)=y$. Then, $T_\text{p}(\gamma(t))=\gamma(1-t)$ and $T_\text{p}(\gamma(1-t))=\gamma(t), \forall t\in[0,1]$. 
	\end{theorem}
	\begin{proof}
		Let $\beta(t)=T_\text{p}(\gamma(t))$. We have to show that $\beta(t)=\gamma(1-t)$. Since $T_\text{p}$ is an isometry, according to (Proposition 16.3, \cite{gallier2012notes}, Chapter 3, p91 \cite{o1983semi}) $T_\text{p}$  maps a shortest length geodesic on $\mathcal{M}$ to a shortest length geodesic on $\mathcal{M}$. Therefore, $\beta(t)$ is also a shortest length geodesic. Now, we have  $\beta(0)=T_\text{p}(\gamma(0))=T_\text{p}(x)=y$ and $\beta(1)=T_\text{p}(\gamma(1))=T_\text{p}(y)=x$. Therefore, $\beta(t)$ is the shortest length geodesic between the points $x$ and $y$ such that $\beta(0)=y$ and $\beta(1)=x$. Since there can only be a single shortest length geodesic curve between two points (except continuous symmetry, like sphere),  both the geodesics $\gamma(t)$ and $\beta(t)$ trace the same path. However, their start and end points are flipped. Therefore, $\beta(t)=\gamma(1-t)\Rightarrow T_\text{p}(\gamma(t))=\gamma(1-t)$. Since the self-isometry is an involution, i.e. $T_\text{p}\circ T_\text{p}(x)=x,\forall x\in\mathcal{M}$, we have $\gamma(t)=T_\text{p}\circ T_\text{p}(\gamma(t))=T_\text{p}(T_\text{p}(\gamma(t)))=T_\text{p}(\gamma(1-t))\Rightarrow T_\text{p}(\gamma(1-t))=\gamma(t)$.\qed
	\end{proof}
	The intuitive is that if a shape is intrinsically symmetric, then the shortest length geodesic curve between any two intrinsically symmetric points is also intrinsically symmetric. This result helps us to determine the sign of the eigenfunctions of the Laplace-Beltrami operator. First, we show that the result $\phi_i\circ T_\text{p}(x)=\pm\phi_i(x)\;\forall x\in\mathcal{M}$ holds true if we restrict the eigenfunctions on the shortest length geodesic curve between the intrinsically symmetric points. The restriction of $\phi_i$ on a curve $\gamma(t)$ is defined as $\phi_i\circ\gamma(t):[0,1]\rightarrow\mathbb{R}$.  Since, $\phi_i\circ T_\text{p}=\pm\phi_i, \forall i$ such that $i$-th eigenvalue is non repeating, each eigenfunction is always either an even (sign$=+1$) or an odd (sign$=-1$) function. Hence, if restriction of the eigenfunction $\phi_i$ on the shortest length  geodesic between the intrinsically symmetric points has sign $+1$ ($-1$), then the sign of  $\phi_i$ is also $+1 \;(-1)$. 
	\begin{proposition}
		\label{pro_1}
		Let  $x,y\in\mathcal{M}$ be two intrinsically symmetric points and $\gamma(t):[0,1]\rightarrow\mathcal{M} $ be the shortest length geodesic curve between the points $x$ and $y$. Then, $
		\phi_i\circ\gamma(t)=\pm\phi_i\circ\gamma(1-t),\forall t\in[0,1].$
	\end{proposition}
	\begin{proof}
		Using  Theorem \ref{th_2}, we proceed as $\phi_i(\gamma(t))=\phi_i(T_\text{p}(\gamma(1-t)))=(\phi_i\circ T_\text{p})(\gamma(1-t))$. We know that $\phi_i\circ T_\text{p}=\pm\phi_i$. Hence, $\phi_i(\gamma(t))=\pm\phi_i(\gamma(1-t))$. \qed 
	\end{proof}
	We apply the above result to find whether an eigenfunction is even or odd as follows. We first determine a set of candidate pairs of intrinsically symmetric points. Then we find the shortest length geodesic curve between each pair. Then, for each eigenfunction $\phi_i$, we determine if the restricted  eigenfunction $\phi_i\circ\gamma(t)$ is an even or an odd function for each pair.
	\subsection{Computation of Eigenfunction of Laplace-Beltrami Operator}
	Let $\mathcal{T}=(\mathcal{V},\mathcal{F},\mathcal{E})$ be a triangle mesh, where $\mathcal{V}$ is the set of $n$ vertices, $\mathcal{F}$ is the set of faces, and $\mathcal{E}$ is the set of edges. We follow the  method in \cite{pinkall1993computing} to find the eigenvalues and the corresponding eigenvectors of the Laplace-Beltrami operator in the discrete settings. The discrete Laplace-Beltrami operator is defined by the matrix $\mathbf{L}=-\mathbf{A}^{-1}\mathbf{M}$. Both $\mathbf{M}$ and $\mathbf{A}$ are of size $n\times n$ and are defined as follows.
	
	\begin{minipage}[c]{0.45\textwidth}
		\begin{center}
			$$
			\mathbf{M}_{j,j^\prime}=\begin{cases}
			\frac{\cot(\alpha_{jj^\prime})+\cot(\beta_{jj^\prime})}{2} &\text{ if } (j,j^\prime)\in\mathcal{E}\\
			-\sum_{j^{\prime\prime}\neq j}\mathbf{M}_{j,j^{\prime\prime}} &\text{ if } j=j^\prime\\
			0&\text{ if } (j,j^\prime)\notin\mathcal{E},
			\end{cases}
			$$
		\end{center}
	\end{minipage}
	\begin{minipage}[c]{0.4\textwidth}
		\begin{center}
			\epsfig{figure=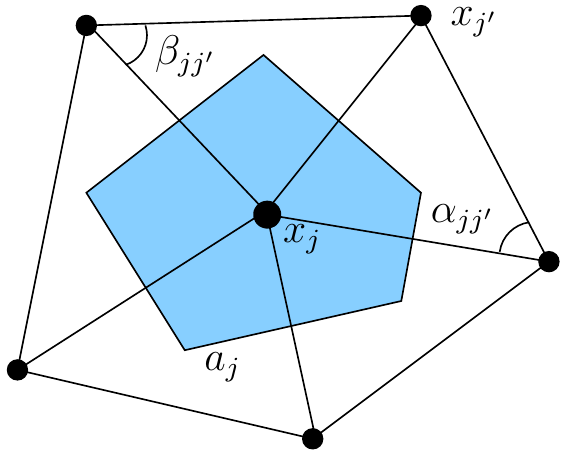,width=0.5\linewidth}
		\end{center}
	\end{minipage}
	\\$\mathbf{A}=\text{diag}(a_1,a_2,\ldots,a_n)$, and, $a_j=\frac{1}{3}\sum_{j^\prime,j^{\prime\prime}:(j,j^\prime,j^{\prime\prime})\in\mathcal{F}}a_{jj^\prime j^{\prime\prime}}$ (the area of the shaded region in the inset figure). Here, $a_{jj^\prime j^{\prime\prime}}$ is the area of the face $(j,j^\prime,j^{\prime\prime})$. In the discrete settings, we denote eigenfunctions by $\boldsymbol{\phi}_i,i\in[k]$, and are the solutions of the generalized eigen-problem $\mathbf{M}\boldsymbol{\phi}_i=-\lambda_i \mathbf{A}\boldsymbol{\phi}_i$. Here, $[k]=\{1,2,\ldots,k\}$. We denote the value of $\boldsymbol{\phi}_i$ at the $j$-th point or vertex by $\boldsymbol{\phi}_i(x_j)$.
	\subsection{Detecting Pairs of Intrinsically Symmetric Points}
	\label{subsec:3_4}
	In order to detect the intrinsic symmetry, according to Theorem \ref{th_1}, we need to find the  matrix $\mathbf{C}$ defined as $\mathbf{C}_{i,i^\prime}=0$, if $i\neq i^\prime$, $\mathbf{C}_{i,i}=+1$, if $\boldsymbol{\phi}_i$ is an even function, and $\mathbf{C}_{i,i}=-1$, if $\boldsymbol{\phi}_i$ is an odd function. According to Proposition \ref{pro_1}, the eigenfunction $\boldsymbol{\phi}_i$ is an even function (odd) if its restriction on the shortest length geodesic between intrinsically symmetric points is also an even function (odd). Therefore, our first task is to find a few accurate candidate pairs of intrinsically symmetric points. 
	\begin{figure}
		\centering
		\stackunder{\epsfig{figure=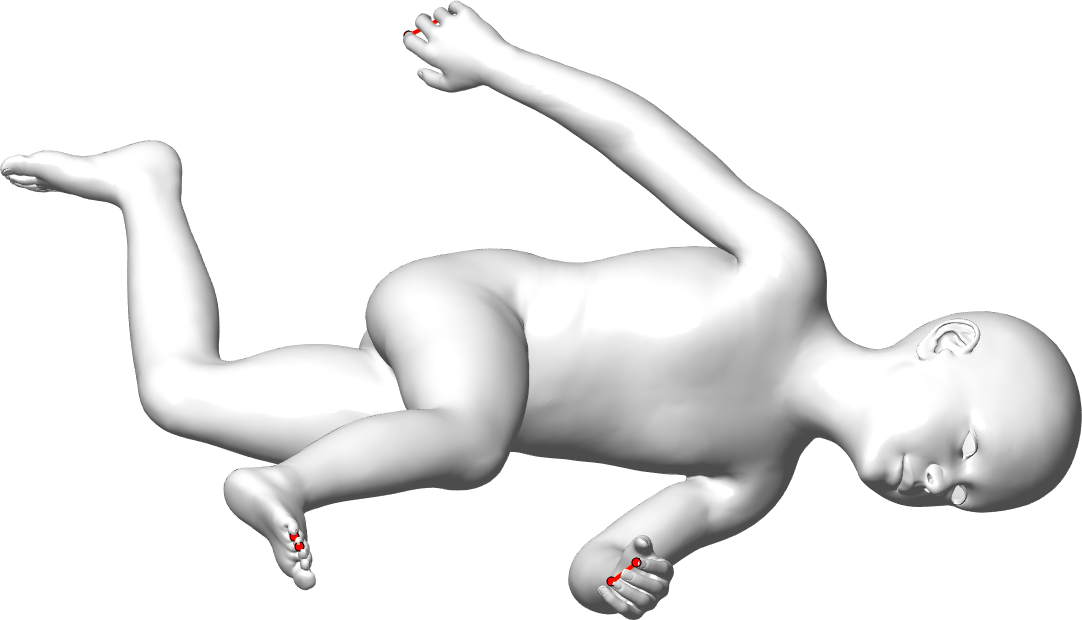,width=0.39\linewidth}}{(a)}
		\stackunder{\epsfig{figure=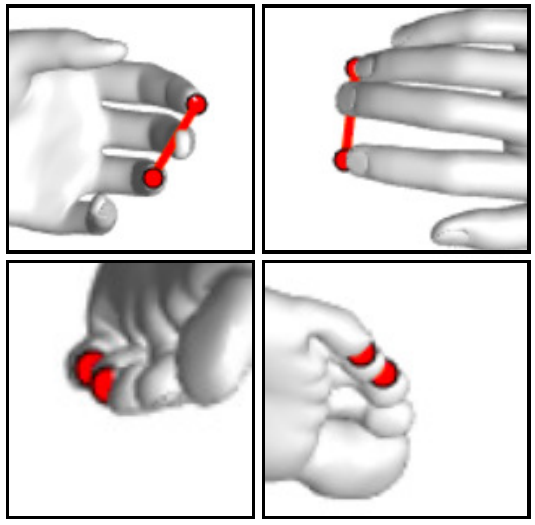,width=0.2\linewidth}}{(b)}
		\stackunder{\epsfig{figure=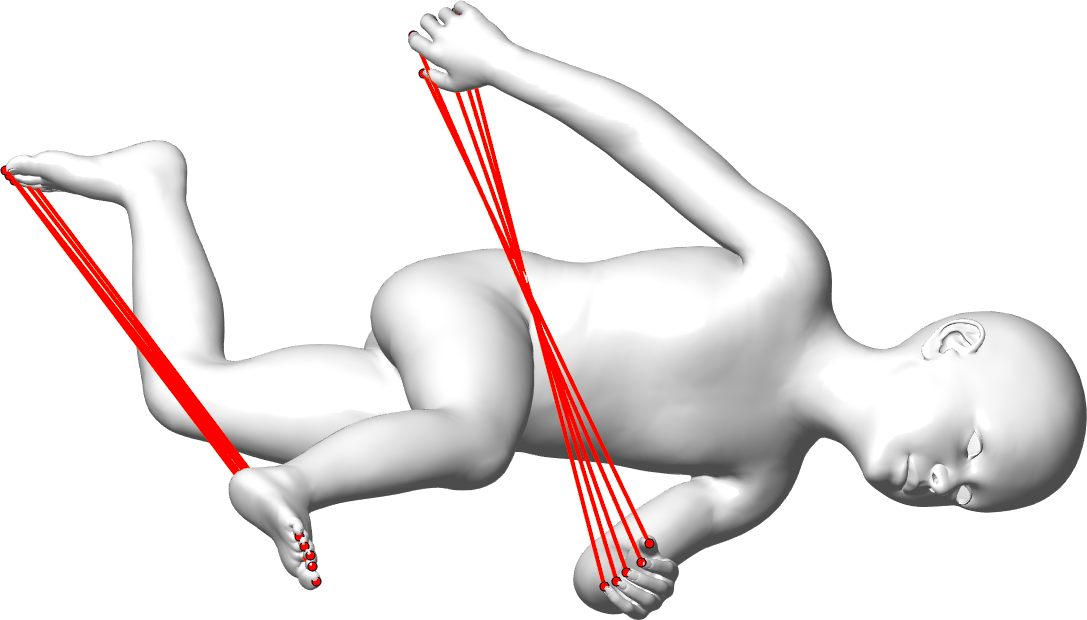,width=0.39\linewidth}}{(c)}
		\caption{Direct HKS  matching vs. restricted HKS matching. (a)-(b) Pairs of intrinsically symmetric points using HKS similarity. It suffers from the fact that if two neighboring points are subjected to a same strength heat source, then their heat diffusions will be similar. (c) Pairs obtained by restricted HKS matching. We observe that the sign of low frequency eigenfunctions on two neighboring points is the same. Therefore, we assign a high cost for pairing two points having the same vectors of signs of eigenfunctions.}
		\label{mspps} 
	\end{figure}
	We find the heat kernel signature (HKS) feature points \cite{sun2009concise} on the given mesh $\mathcal{T}$. HKS feature points are the local maxima of the function $\sum_{i=1}^{k}e^{-\lambda_i t_h}\boldsymbol{\phi}_i^2(x_j)$ defined on the shape $\mathcal{T}$ for all vertices $x_j$, $j\in[n]$. We use $k=13$ for our experiments. Let $\{x_j\}_{j\in\mathcal{I}}$ be the set of HKS feature points, where, $\mathcal{I}\subset[n]$, and $|\mathcal{I}|=d$. We set $t_h=\frac{4\ln10}{\lambda_2}$ as defined in \cite{sun2009concise}. To find the pairs of intrinsically symmetric points, we do not directly match the HKS descriptors. The reason is that this strategy may fail very frequently. In Fig.\ref{mspps}(a), we show  the detected HKS feature points and  the pairs  detected by  direct matching of their HKS descriptors on the Kids model \cite{cosmo2016shrec}. In Fig.\ref{mspps}(b), we show the zoomed pairs of symmetric points on the feet and the hands for better visualization. We observe that the tips of two fingers of the same hand got paired. The reason behind getting such matches is that if two neighboring points are subjected to a same strength heat source, then their heat diffusion processes will be similar. Therefore, their HKS descriptors will be similar. Hence, we have to assign a high cost for pairing two neighboring points. Furthermore, determining if two points are neighbors requires us to find the geodesic distance between each possible pair. This can be a costly process since there can be a large number of  possible pairs. We propose a fast approach to determine if two points are neighbors based on the following observation.\\
	\textbf{Observation 1.} \textit{Let $x_{j}$ and $x_{j^\prime}$ be any two neighboring points in $\mathcal{T}$. Then,  $\mathbf{s}_{j}=\mathbf{s}_{j^\prime}$ for low frequency eigenfunctions, i.e., for small $k$}.\\ Here, $\mathbf{s}_{j}$ and $\mathbf{s}_{j^\prime}$ are the $j$-th and $j^\prime$-th columns of the matrix $\mathbf{S}=\begin{bmatrix} \mathbf{s}_1&\mathbf{s}_2&\ldots&\mathbf{s}_d\end{bmatrix}\in\{-1,+1\}^{k\times d}$. The $i$-th element of the vector $\mathbf{s}_j$ is the sign of the $i$-th eigenfunction on the $j$-th point of the set $\{x_j\}_{j\in\mathcal{I}}$. We give an intuitive understanding of this observation based on the nodal domains of the eigenfunctions. The nodal set of the eigenfunction  $\mathbf{\phi}_i$ is the set $\mathcal{B}_i=\{x\in\mathcal{M}:\phi_i(x)=0\}$. A nodal domain is a component in the set $\mathcal{M}\backslash\mathcal{B}_i$. The set $\mathcal{M}\backslash\mathcal{B}_i$ is the collection of components or segments on the shape which are separated by the set $\mathcal{B}_i$. The value of the eigenfunction $\phi_i$ in any of its nodal domain is either positive or negative \cite{zelditch2013eigenfunctions,reuter2009discrete}. Therefore, if two points lie in the same nodal domain, then the eigenfunction $\phi_i$ will have the same sign on both the points. Now, the neighborliness of two points depends on the size or the area of the nodal domain. According to the Courant's
	nodal domain theorem, the number of nodal domains of $\phi_i$ is less than $i$  \cite{zelditch2013eigenfunctions}. Therefore, the size of the nodal domains remains significantly large for low frequency eigenfunctions. Hence, neighboring points remain in the same nodal domain for all the low frequency eigenfunctions. Therefore, the sign of all low frequency eigenfunctions remains the same on neighboring points. We choose eigenfunctions corresponding to the first $13$ lowest eigenvalues and corresponding eigenvectors for our experiment. 
	
	Hence, we assign a high cost for pairing the points $x_{j}$ and $x_{j^\prime}$,  if their sign vectors $\mathbf{s}_{j}$ and $\mathbf{s}_{j^\prime}$ are the same. Let $\mathbf{H}=\begin{bmatrix} \mathbf{h}_1&\mathbf{h}_2&\ldots&\mathbf{h}_d\end{bmatrix}\in\mathbb{R}^{h\times d}$ be the heat kernel signatures matrix of the detected $d$ HKS feature points, where we choose $h=50$ steps.  We define the affinity matrix $\mathbf{W}\in\mathbb{R}^{d\times d}$ such that $\mathbf{W}_{j,j^\prime}=\|\mathbf{h}_j-\mathbf{h}_{j^\prime}\|_2+q\psi(\|\mathbf{s}_j-\mathbf{s}_{j^\prime}\|_2),\forall(j,j^\prime)\in[d]\times[d],j\neq j^\prime$ and $\mathbf{W}_{j,j}=q, \forall j\in[d]$.  Here $\psi(t)=0$, if $t>0$ and $\psi(t)=1$, if $t=0$, and $q$ is any large positive constant. We now pair these points such that if $x_{j^\prime}$ is intrinsically symmetric point of the point $x_j$, then $x_{j}$ should be the intrinsically symmetric point of the point $x_{j^\prime}$. We achieve this by representing the matching by a matrix $\mathbf{\Pi}\in\{0,1\}^{d\times d}$, where $\mathbf{\Pi}_{j,j^\prime}=1$ and $\mathbf{\Pi}_{j^\prime,j}=1$, if the points $x_j$ and $x_{j^\prime}$ form a pair and 0, otherwise.  Now, we enforce the constraints $\mathbf{\Pi}\mathbf{1}=\mathbf{1}$ and $\mathbf{\Pi}^\top\mathbf{1}=\mathbf{1}$ to achieve one-to-one matching, where $\mathbf{1}$ is a vector of size $d$ with all elements equal to 1. We get many points which can not be paired. Therefore, we cap the number of pairs by $c$ which we represent by the constraint $\mathbf{1}^\top\mathbf{\Pi}\mathbf{1}=2c$. Further, to make it feasible, we modify the one-to-one matching constraints to $\mathbf{\Pi}\mathbf{1}\leq\mathbf{1}$ and $\mathbf{\Pi}^\top\mathbf{1}\leq\mathbf{1}$. Now, we frame the problem of pairing the points in the below optimization problem.
	\begin{eqnarray}
	\underset{\mathbf{\Pi}\in\{0,1\}^{d\times d}}{\min}\sum_{j=1}^d\sum_{j^\prime=1}^d\mathbf{\Pi}_{j,j^\prime}\mathbf{W}_{j,j^\prime} 
	\text{, subject to}&\mathbf{\Pi}\mathbf{1}\leq\mathbf{1},\;\mathbf{\Pi}^\top\mathbf{1}\leq\mathbf{1},\;\mathbf{1}^\top\mathbf{\Pi}\mathbf{1}=2c.
	\label{eq:1}
	\end{eqnarray}
	We note that,  problem  (\ref{eq:1}) is equivalent to the below linear assignment problem.
	\begin{eqnarray}
	\underset{\boldsymbol{\pi}\in\{0,1\}^{d^2\times 1}}{\min}\mathtt{vec}(\mathbf{W})^\top\boldsymbol{\pi}
	\text{, subject to}&\mathbf{C}_1\boldsymbol{\pi}\leq\mathbf{1},\;\mathbf{C}_2\boldsymbol{\pi}\leq\mathbf{1},\;\mathbf{c}_3^\top\boldsymbol{\pi}=2c.
	\label{eq:2}
	\end{eqnarray}
	Here, the vector $\boldsymbol{\pi}=\mathtt{vec}(\mathbf{\Pi})$, $\mathbf{C}_2=\mathbf{1}^\top\otimes\mathbf{I}$, $\mathbf{c}_3$ is the vector of size $d^2\times 1$ with all elements equal to 1, and $\mathbf{I}$ is the identity matrix of size $d\times d$. The matrix $\mathbf{C}_1$ is $d\times d^2$ matrix and defined such that the $j$-th row is the circular shift, by $j d$ elements on the right, of the row vector of size $1\times d^2$ with the first $d$ elements equal to $1$ and the last $d^2-d$ elements equal to 0. The time complexity of this problem is exponential in the number of variables. However, the size of our problem is very small. In all our experiments $d\leq25$. We use the MATLAB function \texttt{intlinprog} to solve this problem which takes $\approx0.03$ seconds.  
	\subsection{Determining the Sign of Eigenfunctions}
	Proposition \ref{pro_1} states that the eigenfunction $\boldsymbol{\phi}_i$ is an even (odd) function, if its restriction $\boldsymbol{\phi}_i\circ\gamma_{j}(t):[0,1]\rightarrow \mathbb{R}$ on the shortest length geodesic  $\gamma_j(t)$ between any two intrinsically symmetric points $x_{j}$ and $x_{j^\prime}$ is an even (odd) function. Let $\{(x_j,x_{j^\prime})\}_{j=1}^{c}$, be the set of detected pairs of intrinsically symmetric points. We find the shortest length geodesic curve between two intrinsically symmetric points using \cite{surazhsky2005fast} with approximate setting (Dijkstra's algorithm), since the exact geodesic curve may not pass through the vertices of the mesh which may require us to perform interpolation for calculating the values of $\boldsymbol{\phi}_i\circ\gamma_{j}(t)$  for $\gamma_{j}(t)\notin\mathcal{V}$. Let $\mathbf{p}_{ij}$ be the restriction (vector of size equal to the number of vertices in the geodesic) of the eigenfunction $\boldsymbol{\phi}_i$ on the shortest length geodesic curve between the intrinsically symmetric points $x_j$ and $x_{j^\prime}$. Then, the sign $s_i$ of eigenfunction $\boldsymbol{\phi}_i$ is equal to $+1$, if $\sum_{j=1}^c\mathbf{p}_{ij}^\top\texttt{flip}(\mathbf{p}_{ij})>0$ and equal to $-1$, if  $\sum_{j=1}^c\mathbf{p}_{ij}^\top\texttt{flip}(\mathbf{p}_{ij})<0$. We do not consider the eigenfunction $\boldsymbol{\phi}_i$ if $\sum_{j=1}^c\mathbf{p}_{ij}^\top\texttt{flip}(\mathbf{p}_{ij})=0$. Equivalently, we define the diagonal entries of the functional correspondence matrix $\mathbf{C}$ as $\mathbf{C}_{i,i}=s_i$. In Fig. \ref{geo_d}(a) and (c), we show the eigenfunctions $\boldsymbol{\phi}_2$ and $\boldsymbol{\phi}_3$, respectively, with the geodesic curves for two pairs of detected intrinsically symmetric points on a shape from Kids dataset \cite{cosmo2016shrec}. In Fig. \ref{geo_d}(b) and (d), we show the functions $\boldsymbol{\phi}_i\circ\gamma_{j}(t)$ for $i=2,3$ and $j=1,2$. We observe that $\mathbf{C}_{2,2}=-1$ and $\mathbf{C}_{3,3}=+1$.
	
	One can directly use the values of an eigenfunction on the intrinsically symmetric points to find the sign instead of checking it on the geodesic between these points. However, this approach could be sensitive to the noise. If the value of an eigenfunction at the feature point has changed due to noise, then the point-based method will fail. Whereas, it is less likely that due to the noise the value of an eigenfunction will be changed at all the points on the geodesic. Our geodesic based method will detect the sign correctly due to  averaging of signs. 
	\begin{figure}[t!]
		\centering
		\stackunder{\epsfig{figure=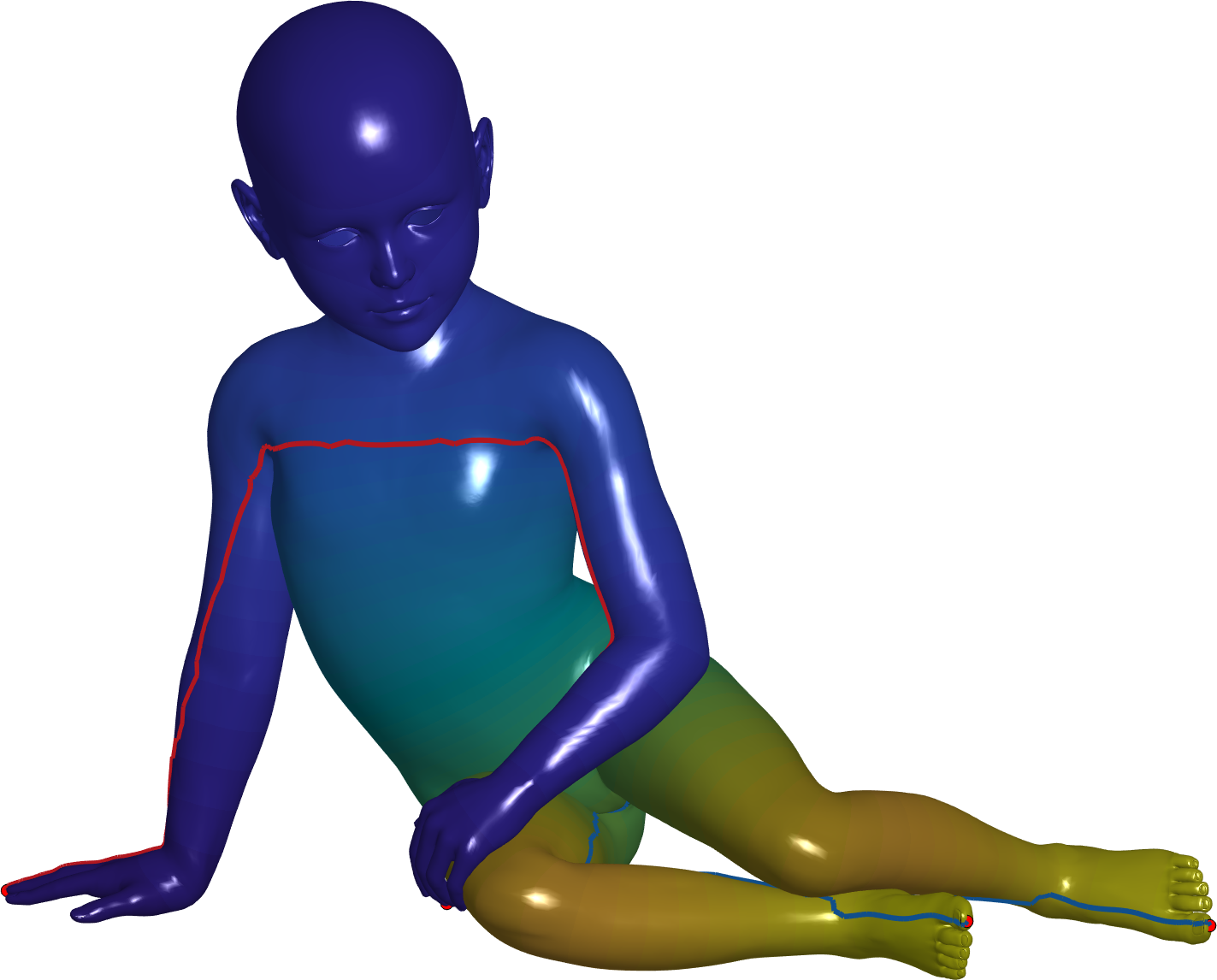,width=0.24\linewidth}}{(a)}
		\stackunder{\epsfig{figure=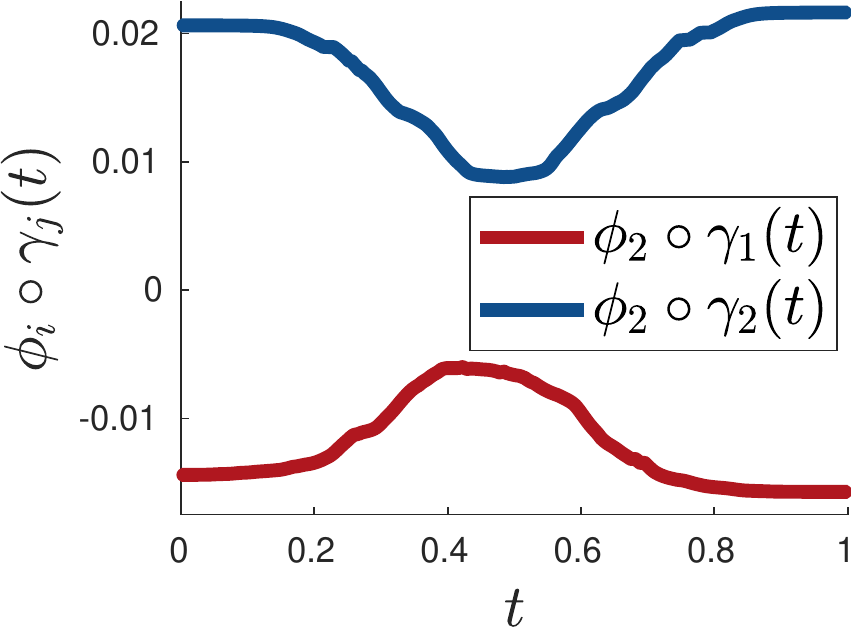,width=0.24\linewidth}}{(b)}
		\stackunder{\epsfig{figure=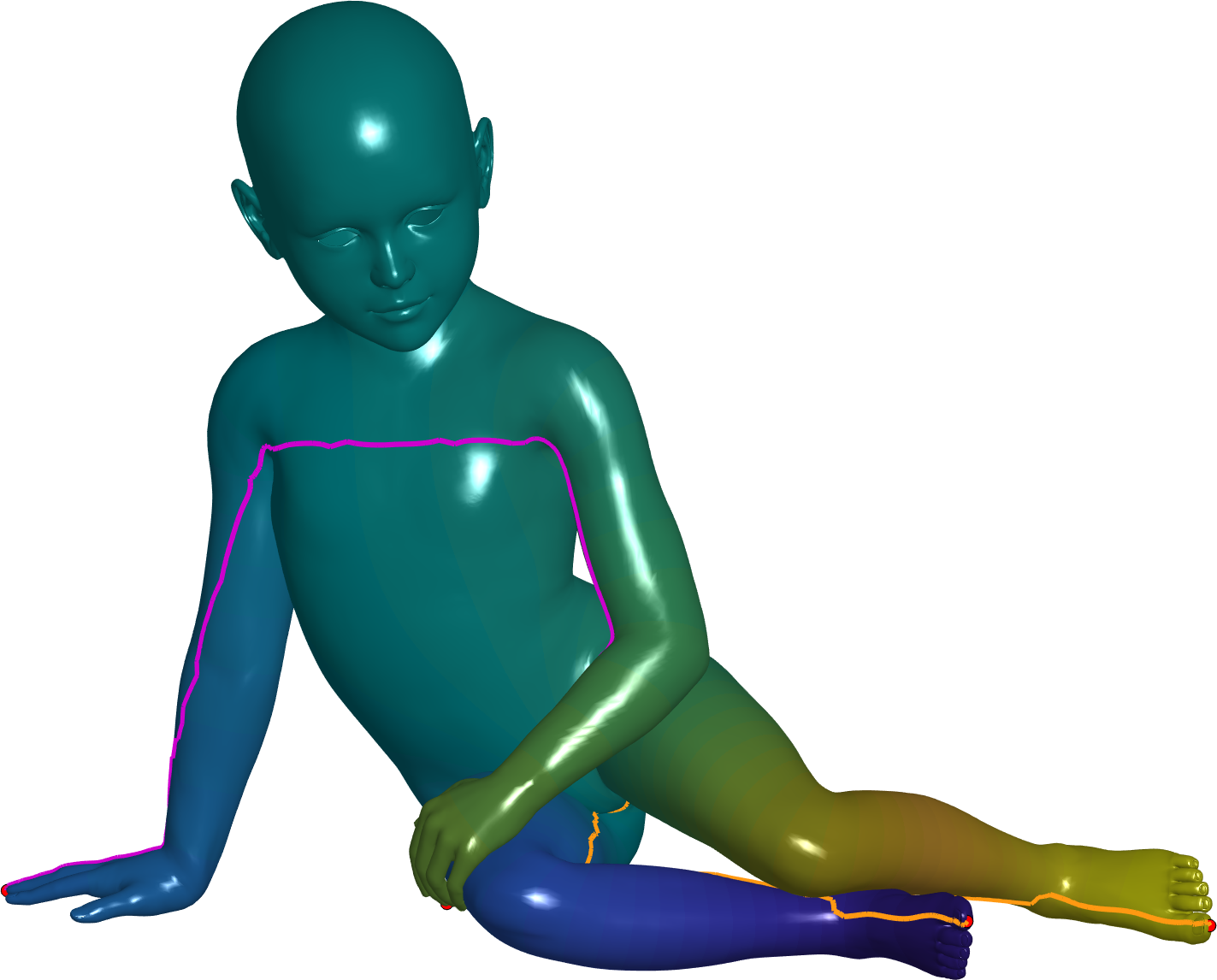,width=0.24\linewidth}}{(c)}
		\stackunder{\epsfig{figure=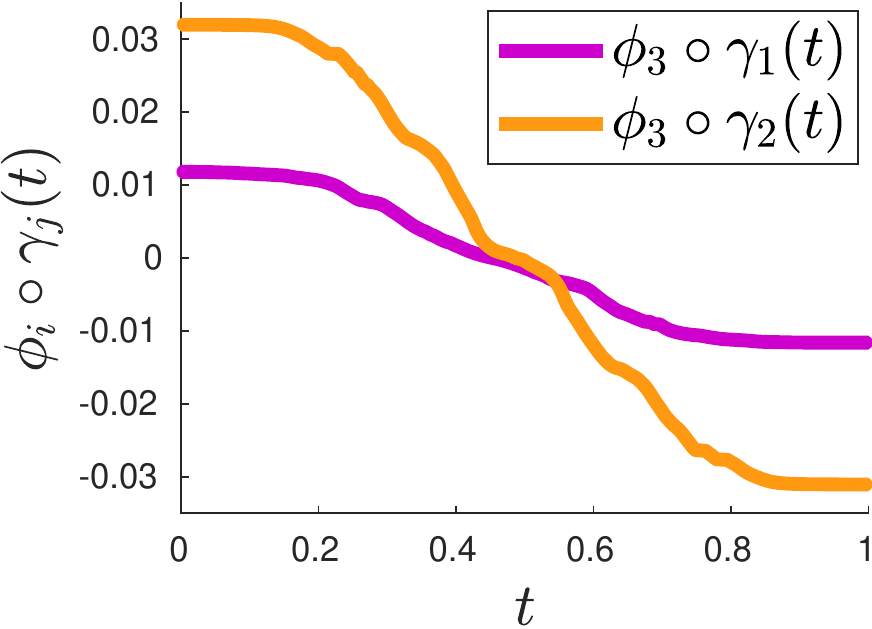,width=0.24\linewidth}}{(d)}
		\caption{(a)-(b): The eigenfunction $\boldsymbol{\phi}_2$ on the Kids model which is an even function and its restrictions $\boldsymbol{\phi}_2\circ\gamma_j(t)$, $j=1,2$, shown in red and blue colors (on hands and feet), which are also even functions. (c)-(d): The eigenfunction $\boldsymbol{\phi}_3$ on the Kids model which is an odd function and its restrictions $\boldsymbol{\phi}_3\circ\gamma_j(t)$, $j=1,2$, shown in purple and orange colors (on hands and feet), which are also odd functions.}
		\label{geo_d}
	\end{figure}
	\subsection{Correcting the Eigenfunctions}
	The models of the benchmark datasets are obtained by applying an
	imperfect isometry, so the theory only holds approximately. Furthermore, some of the triangles may not be Delaunay triangles and the eigenfunctions are sensitive to the change in the triangulation of the mesh. Therefore, all the eigenfunctions may not be perfectly even or odd which may give the erroneous symmetry detected in Section \ref{subsec:disc}. Consider Fig. \ref{geo_imp}(a), where the eigenfunction $\boldsymbol{\phi}_{10}$ is not perfectly even on the legs. We transform the eigenfunctions such that they preserve the pairs of intrinsically symmetric functions. We extend the framework in \cite{kovnatsky2013coupled}.
	Let $\mathbf{\Phi}=\begin{bmatrix}\boldsymbol{\phi}_1&\boldsymbol{\phi}_2&\ldots&\boldsymbol{\phi}_k\end{bmatrix}\in\mathbb{R}^{n\times k}$, and $\mathbf{D}=\text{diag}(\lambda_1,\lambda_2,\ldots,\lambda_k)\in\mathbb{R}^{k\times k}$.  Let $\mathbf{\Phi}\mathbf{R}$ be the transformed basis obtained by applying the linear operator $\mathbf{R}$ on the basis $\mathbf{\Phi}$. Then, we impose the constraints $\mathbf{R}^\top\mathbf{D}\mathbf{R}=\mathbf{D}$ and  $\text{off}(\mathbf{R}^\top\mathbf{D}\mathbf{R})=0$ so that  the new eigenfunctions admit to the original eigenfunction decomposition problem as proposed in \cite{kovnatsky2013coupled}, where $\text{off}(\mathbf{M})=\sum_{j}\sum_{j^\prime:j^\prime\neq j}\mathbf{M}_{j,j^\prime}^2$ for any matrix $\mathbf{M}$. 
	
	Now, let $f_j,g_j:\mathcal{M}\rightarrow\mathbb{R}$ be two functions such that $f_j$ and $g_j$ are intrinsic images of each other. That is, $f_j\circ T_\text{p}(x)=g_j(x)$ and $g_j\circ T_\text{p}(x)=f_j(x)$ are equivalent. Let $\mathbf{f}_j\in\mathbb{R}^{n}$ and $\mathbf{g}_j\in\mathbb{R}^{n}$ be the discrete versions of $f_j$ and $g_j$, respectively. Let $\mathbf{R}^\top\mathbf{\Phi}^\top\mathbf{f}_j$ and $\mathbf{R}^\top\mathbf{\Phi}^\top\mathbf{g}_j$ be the representations of the functions $\mathbf{f}_j$ and $\mathbf{g}_j$ in the transformed basis $\mathbf{\Phi}\mathbf{R}$, respectively. We want the transformed basis $\mathbf{\Phi}\mathbf{R}$ such that  $\mathbf{R}^\top\mathbf{\Phi}^\top\mathbf{f}_j=\mathbf{C}\mathbf{R}^\top\mathbf{\Phi}^\top\mathbf{g}_j$.  Let $\mathbf{F}=\begin{bmatrix}\mathbf{f}_1&\ldots&\mathbf{f}_c&\mid&\mathbf{g}_1&\ldots&\mathbf{g}_c \end{bmatrix}\in\mathbb{R}^{n\times 2c}$ and  $\mathbf{G}=\begin{bmatrix}\mathbf{g}_1&\ldots&\mathbf{g}_c&\mid&\mathbf{f}_1&\ldots&\mathbf{f}_c \end{bmatrix}\in\mathbb{R}^{n\times 2c}$ be the matrices representing $2c$ (bidirectional) pairs of intrinsically symmetric functions. We formulate the below optimization framework to find the transformation matrix $\mathbf{R}$. 
	\begin{eqnarray}
	\nonumber&\underset{\mathbf{R}}{\min} \text{ off}(\mathbf{R}^\top\mathbf{D}\mathbf{R})+\|\mathbf{R}^\top\mathbf{D}\mathbf{R}-\mathbf{D}\|_\text{F}^2\\ \text{subject to}&\mathbf{R}^\top\mathbf{\Phi}^\top\mathbf{F}=\mathbf{C}\mathbf{R}^\top\mathbf{\Phi}^\top\mathbf{G},\;
	\mathbf{R}^\top\mathbf{R}=\mathbf{I},\text{det}(\mathbf{R})=+1,\mathbf{R}\in\mathbb{R}^{k\times k}.
	\label{eq:3}
	\end{eqnarray}    
	Here, $\mathbf{R}^\top\mathbf{R}=\mathbf{I}$ follows from the fact that the transformed basis $\mathbf{\Phi} \mathbf{R}$ is an orthogonal basis. Here, the set  $\{\mathbf{R}\in\mathbb{R}^{k\times k}:\mathbf{R}^\top\mathbf{R}=\mathbf{I},\text{det}(\mathbf{R})=+1\}$ is the special orthogonal group $\mathcal{SO}(k)$. Hence, we solve the below optimization problem.
	\begin{eqnarray}
	\underset{\mathbf{R}\in \mathcal{SO}(k)}{\min}\text{ off}(\mathbf{R}^\top\mathbf{D}\mathbf{R})+\|\mathbf{R}^\top\mathbf{D}\mathbf{R}-\mathbf{D}\|_\text{F}^2+\mu\|\mathbf{R}^\top\mathbf{\Phi}^\top\mathbf{F}-\mathbf{C}\mathbf{R}^\top\mathbf{\Phi}^\top\mathbf{G}\|_\text{F}^2.
	\label{eq:4}
	\end{eqnarray} 
	\begin{figure}[t!]
		\centering
		\stackunder{\epsfig{figure=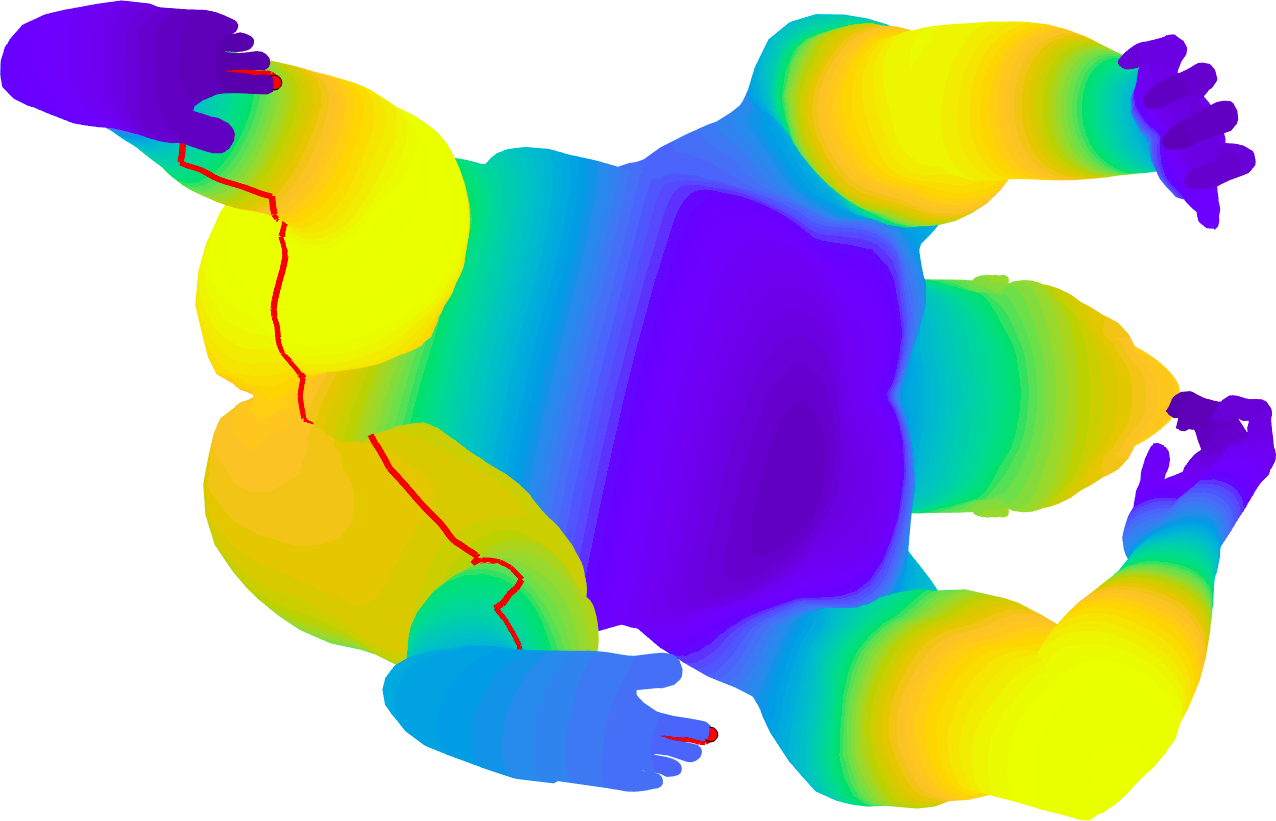,width=0.32\linewidth}}{(a) $\boldsymbol{\phi}_{10}$ before correction}
		\stackunder{\epsfig{figure=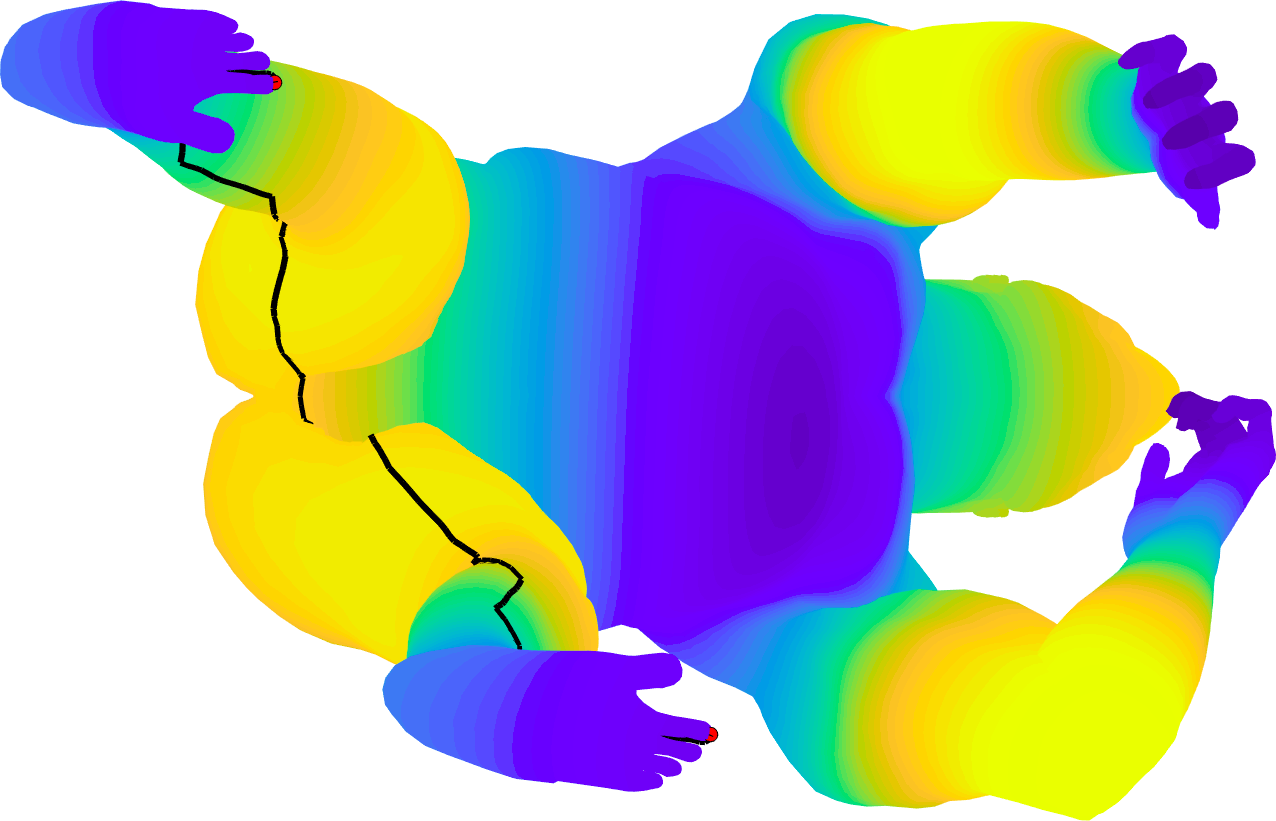,width=0.32\linewidth}}{(b) $\boldsymbol{\phi}_{10}$ after correction}
		\stackunder{\epsfig{figure=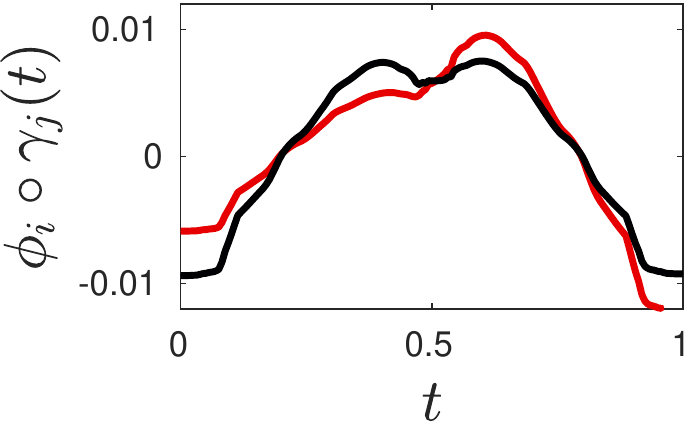,width=0.32\linewidth}}{(c) Restricted eigenfunctions}
		\stackunder{\epsfig{figure=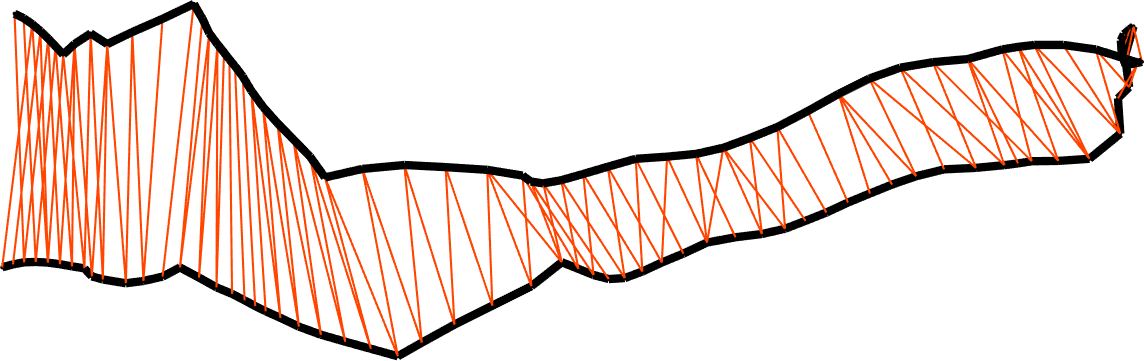,width=0.40\linewidth}
			\epsfig{figure=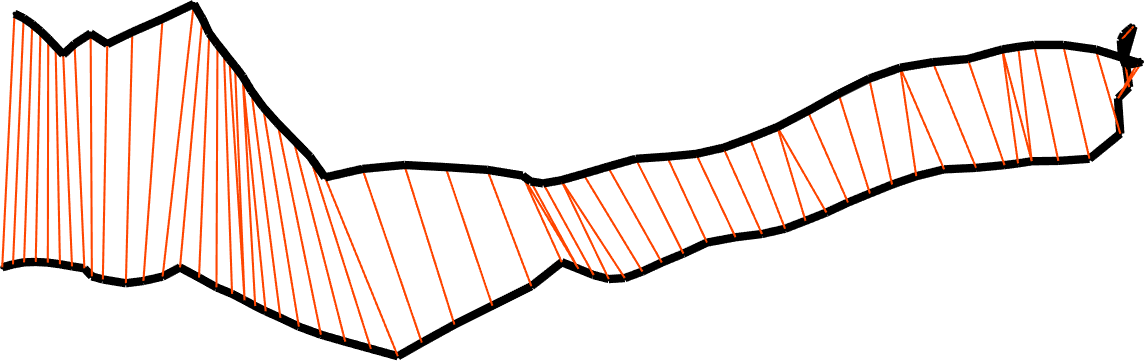,width=0.40\linewidth}}{(d) Symmetry on the geodesic on the legs before (left) and after (right) correction.}
		\caption{Visualization of the eigenfunction correction.}
		\label{geo_imp}
	\end{figure}
	We use the Riemannian-trust-region method, proposed in \cite{absil2007trust,absil2009optimization}, to solve this optimization problem. We use the \texttt{manopt} toolbox \cite{boumal2014manopt} for this purpose. We provide the Riemannian gradient and Hessian of this cost function in the supplementary material. We empirically found the optimal $\mu$ to be equal to $1$ in our experiments. We choose the functions $\mathbf{f}_j$ and $\mathbf{g}_j$ such that, $\mathbf{f}_j=1$ at the point $x_j$ and 0 everywhere else, and  $\mathbf{g}_j=1$ at the point $x_{j^\prime}$ and 0 everywhere else. Here, $x_j$ and $x_{j^\prime}$ are intrinsically symmetric points. In Fig. \ref{geo_imp}(a) and (b), we show the effect of correction on ($\boldsymbol{\phi}_{10}$). We observe that $\boldsymbol{\phi}_{10}$, which was not perfectly symmetric on the legs and the belly, becomes more symmetric. The large blue patch on the belly also got moved to center which was more towards left before correction. Here, the value of eigenfunction is color encoded, more blue implies more negative and more yellow implies more positive.  In Fig. \ref{geo_imp}(c), we show the restriction of $\boldsymbol{\phi}_{10}$ on the geodesic between the two symmetric points which becomes more symmetric after the correction. 
	\subsection{Dense Intrinsically Symmetric Correspondence}
	\label{subsec:disc}
	Let $\mathbf{f}_{j}$ be the function such that $\mathbf{f}_{j}=1$ at $x_{j}$ and 0 elsewhere. Similarly, let $\mathbf{g}_{j^\prime}$ be the function such that $\mathbf{g}_{j^\prime}=1$ at $x_{j^\prime}$ and 0 elsewhere. Let $\mathbf{R}^\top\mathbf{\Phi}^\top\mathbf{f}_{j}$ and $\mathbf{R}^\top\mathbf{\Phi}^\top\mathbf{g}_{j^\prime}$ be their basis representation. Then, if the point $x_{j}$ and $x_{j^\prime}$ are intrinsically symmetric then  $\mathbf{R}^\top\mathbf{\Phi}^\top\mathbf{f}_{j}=\mathbf{C}\mathbf{R}^\top\mathbf{\Phi}^\top\mathbf{g}_{j^\prime}$. Which is equivalent to $\mathbf{R}^\top\mathbf{\Phi}^\top\mathbf{F}=\mathbf{C}\mathbf{R}^\top\mathbf{\Phi}^\top\mathbf{G}$ if we consider all points. Now, if $\mathbf{F}$ is equal to the identity matrix of size $n\times n$, then $\mathbf{G}_{j,j^\prime}=1$, if point $x_j$ and $x_{j^\prime}$ form a pair of intrinsically symmetric points, and 0 otherwise. Now following \cite{ovsjanikov2012functional}, the intrinsically symmetric point of $x_j$ is the nearest neighbor of the $j$-th column of the matrix $\mathbf{R}^\top\mathbf{\Phi}^\top$ among the columns of the matrix $\mathbf{C}\mathbf{R}^\top\mathbf{\Phi}^\top$.  The obtained correspondences are continuous as shown in \cite{ovsjanikov2012functional}. Our method is invariant to the ordering of the eigenfunction since the sign of $\boldsymbol{\phi}_i$ and $\mathbf{C}_{i,i}$ only depend on the eigenfunction $\boldsymbol{\phi}_i$. In Fig. \ref{geo_imp}(d), we show the detected symmetry on a geodesic on legs before and after correction. 
	\begin{table}[t!]
		\begin{minipage}[b]{0.50\linewidth}
			\centering
			\caption{The total time for computing intrinsic symmetry for the methods MT \cite{kim2010mobius}, BIM\cite{kim2011blended}, OFM \cite{liu2015properly}, GRS \cite{wang2017group}, and the proposed approach on the TOSCA dataset \cite{bronstein2008numerical}.}
			\begin{tabular}{c c c c c c}
				\hline             
				& MT& BIM& OFM& GRS&Our\\ \hline
				Time (min) &- &360&60&24&\textbf{8}\\\hline
			\end{tabular}
			\label{tab:time}
		\end{minipage}
		\begin{minipage}[b]{0.50\linewidth}
			\centering
			\caption{The correspondence rates and mesh rates for the methods MT \cite{kim2010mobius}, BIM \cite{kim2011blended}, OFM \cite{liu2015properly}, GRS \cite{wang2017group}, and the proposed approach on the SCAPE dataset \cite{anguelov2005scape}.}
			\begin{tabular}{c c c c c c}
				\hline             
				& MT& BIM& OFM& GRS&Our\\ \hline
				Corr rate (\%)  &82.0 &84.8&91.7&94.5&\textbf{97.5}\\
				Mesh rate (\%)  &71.8 &76.1&97.2&98.6&\textbf{100}\\\hline
			\end{tabular}
			\label{tab:res_scape}
		\end{minipage}
	\end{table}
	\begin{table}[t!]
		\centering
		\caption{The correspondence rates and mesh rates for the methods MT \cite{kim2010mobius}, BIM \cite{kim2011blended}, OFM \cite{liu2015properly}, GRP \cite{wang2017group}, and the proposed approach on the TOSCA dataset \cite{bronstein2008numerical}.}
		\label{tab:res_tosca}
		\begin{tabular}{c c c c c c c c c c c c}
			\hline
			&\multicolumn{5}{c}{Corr rate (\%)} & &\multicolumn{5}{ c }{Mesh Rate (\%)} \\ \cline{2-6} \cline{8-12}
			& MT& BIM& OFM& GRS&Our&$\:\:\:\:\:\:$&MT& BIM& OFM& GRS&Our\\ \hline
			Cat     & 66.0  & 93.7   & 90.9   &\textbf{ 96.5 }  & \textbf{95.6 } & $\:\:\:\:\:\:$& 54.6  & 90.9    & 90.9  & 100   & 100 \\ 
			Centaur & 92.0  & \textbf{100}    &  96.0  & 92.0   & \textbf{100}   &$\:\:\:\:\:\:$ & 100   & 100     &  100  & 100   & 100 \\ 
			David   & 82.0  & \textbf{97.4 }  & 94.8   & 92.5   & \textbf{96.2}  &$\:\:\:\:\:\:$ & 57.1  & 100     &  100  & 100   & 100 \\ 
			Dog     & 91.0  & \textbf{100}    & 93.2   & 97.4   & \textbf{98.8 }  &$\:\:\:\:\:\:$ & 88.9  & 100     & 88.9  & 100   & 100  \\ 
			Horse   & 92.0  & 97.1   & 95.2   & \textbf{99.4}   & \textbf{97.3}  &$\:\:\:\:\:\:$ & 100   & 100     & 87.5  & 100   & 100 \\ 
			Michael & 87.0  & \textbf{98.9}   & 94.6   & 91.4   & \textbf{96.5}  &$\:\:\:\:\:\:$ & 75    & 100     &  100  & 100   &100  \\
			Victoria& 83.0  & \textbf{98.3}   & \textbf{98.7}   & 95.5   & 96.2  &$\:\:\:\:\:\:$ & 63.6  & 100     & 100   & 100   &   100\\ 
			Wolf    & 100   & 100    & 100    & \textbf{100}   &\textbf{100}   &$\:\:\:\:\:\:$ & 100   & 100     & 100   & 100   & 100 \\ 
			Gorilla & -     & 98.9   & 98.9   & \textbf{100}    & \textbf{100}   &$\:\:\:\:\:\:$ & -     & 100     & 100   & 100   & 100 \\ \hline
			Average & 85.0  & \textbf{98.0 }  & 95.1   &  94.5  &\textbf{ 97.8 } &$\:\:\:\:\:\:$ & 76    & 98.7    & 92.6  & \textbf{100}   & \textbf{100} \\ \hline
		\end{tabular}
	\end{table}
	\section{Results and Evaluation}
	\subsection{Time Complexity}
	Let $n$ be the number of vertices and $k$ be the number of eigenfunctions used. The feature points are the local maximums of $\sum_{i=1}^{k}e^{-\lambda_i t_h}\boldsymbol{\phi}_i^2(x_j)$. It requires us to find 2-ring neighborhoods of each vertex. We use the half-edge data structure which requires $O(1)$ time. Hence, the overall time for finding the feature points is $O(n)$. The optimization problem in Eq. (\ref{eq:4}) takes  $O(nk^2)$ when solved using Riemannian trust region method. We use the ANN library \cite{arya1998optimal} to find the nearest neighbor for each column of the matrix $\mathbf{R}^\top\mathbf{\Phi}^\top\in\mathbb{R}^{k\times n}$ among the columns of the matrix $\mathbf{C}\mathbf{R}^\top\mathbf{\Phi}^\top\in\mathbb{R}^{k\times n}$ which takes time  $O(kn\log(n))$. Hence, the time complexity is $O(kn\log(n))+O(n)+O(nk^2)\approx O(kn\log(n))$, since $k<<n$. In our experiments $k=13$ (empirical) and $n\approx 15000$. The time complexity of computing the $k$ smallest eigenvalues and corresponding eigenvectors of symmetric matrix is $O(n^2k)$ which is common to all spectral decomposition based methods.
	\subsection{Comparison}
	\textbf{Evaluation Metrics}. We use the following evaluation metrics to compare the results of our method to that of the state-of-the-art methods as defined in \cite{kim2010mobius}. \textbf{Correspondence rate}: Let $(x_j,x_{j^\prime}^\text{g})$ be the ground truth correspondence and \begin{figure}[t!]
		\begin{center}
			\epsfig{figure=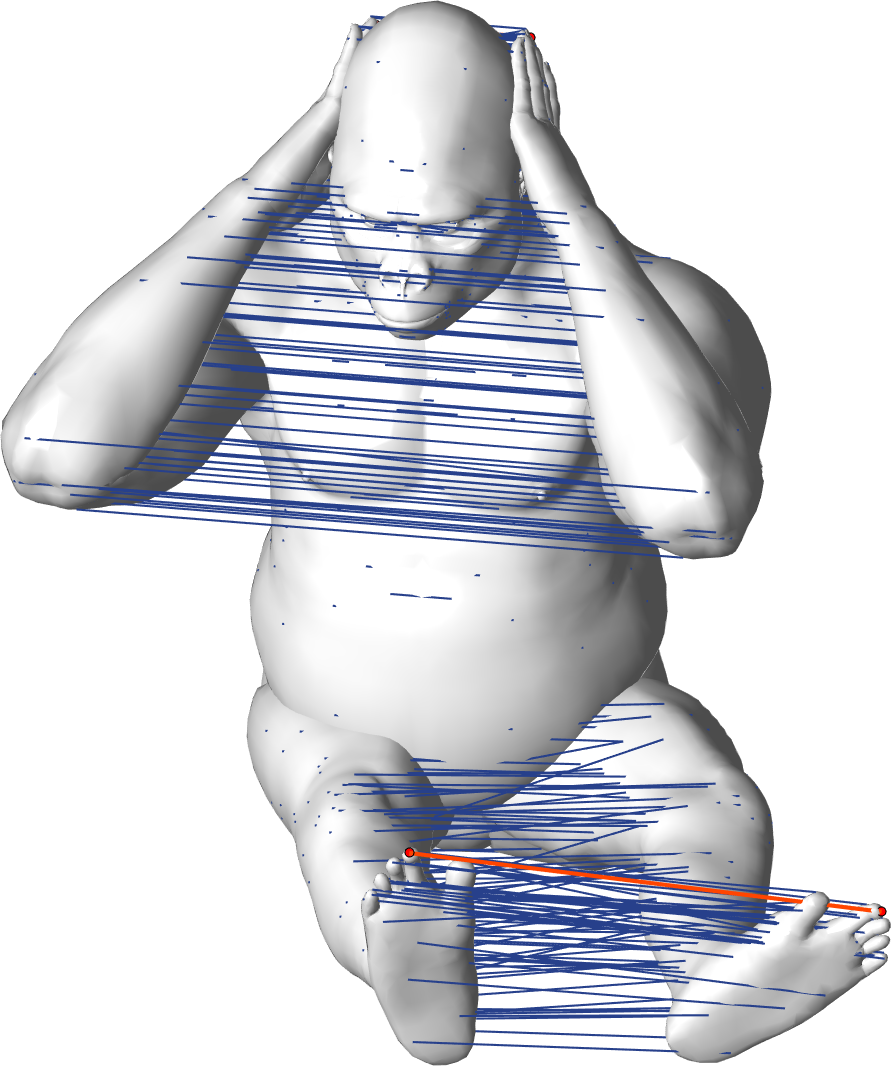,width=0.18\linewidth}
			\epsfig{figure=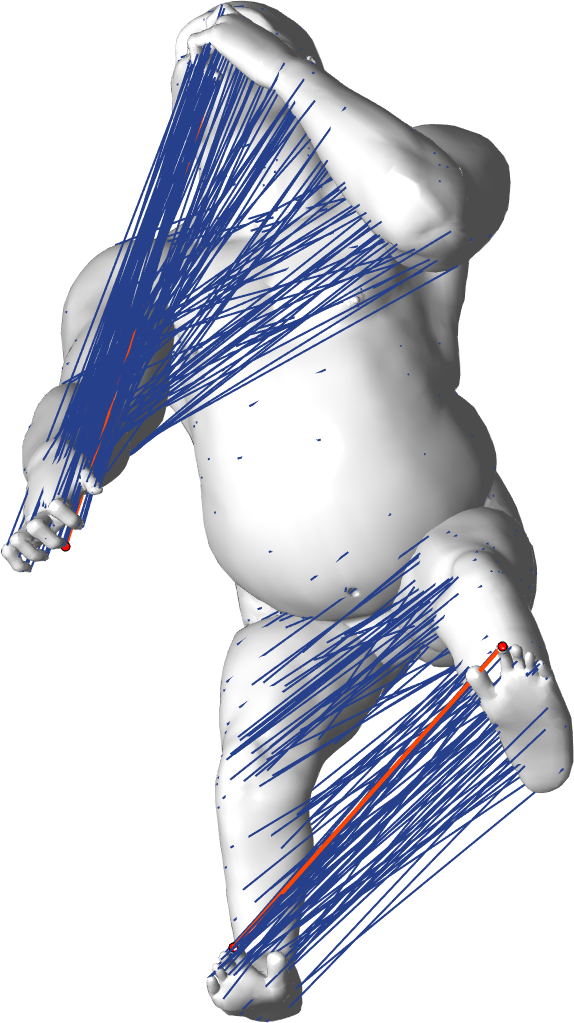,width=0.12\linewidth}
			\epsfig{figure=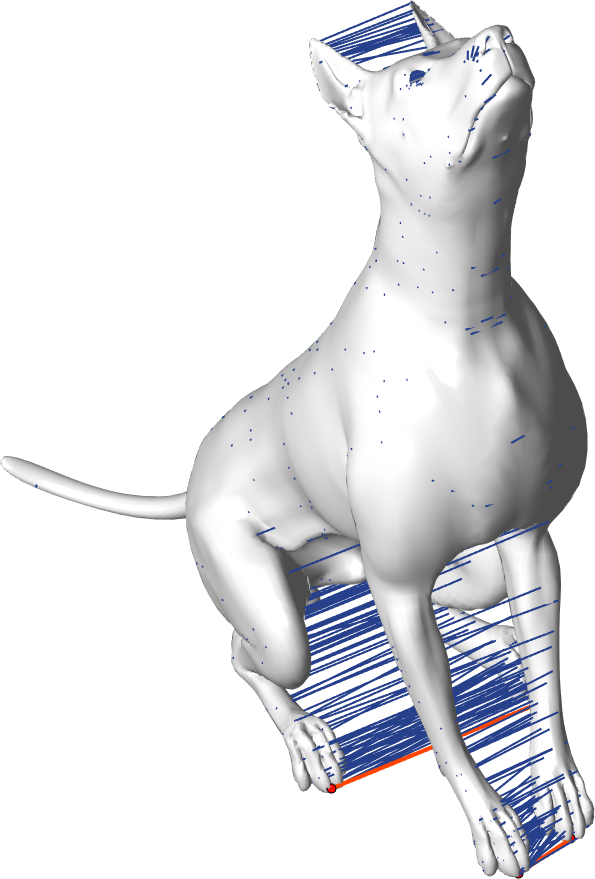,width=0.14\linewidth}
			\epsfig{figure=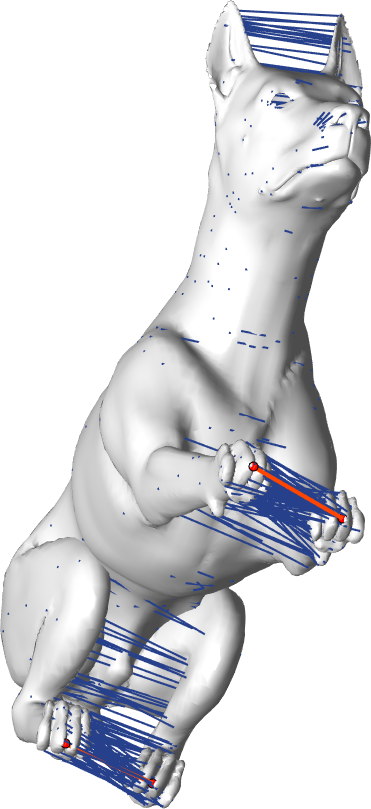,width=0.1\linewidth}
			\epsfig{figure=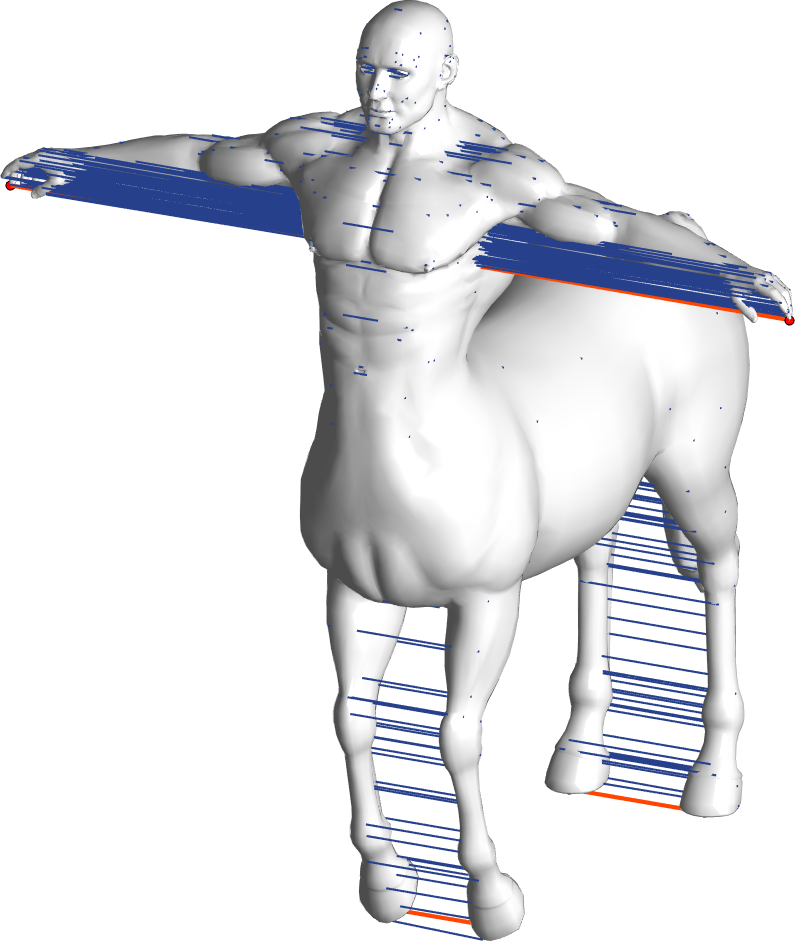,width=0.18\linewidth}
			\epsfig{figure=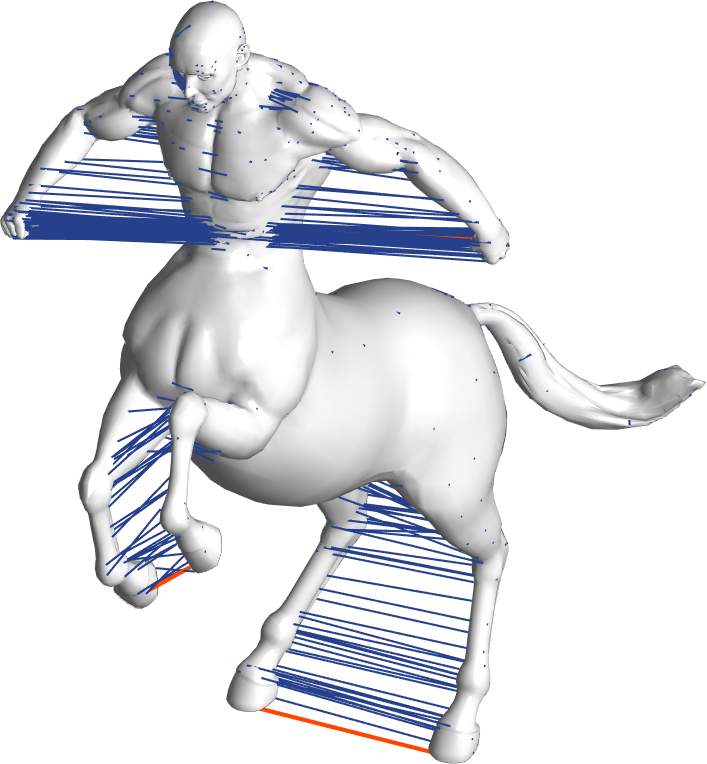,width=0.20\linewidth}
			\epsfig{figure=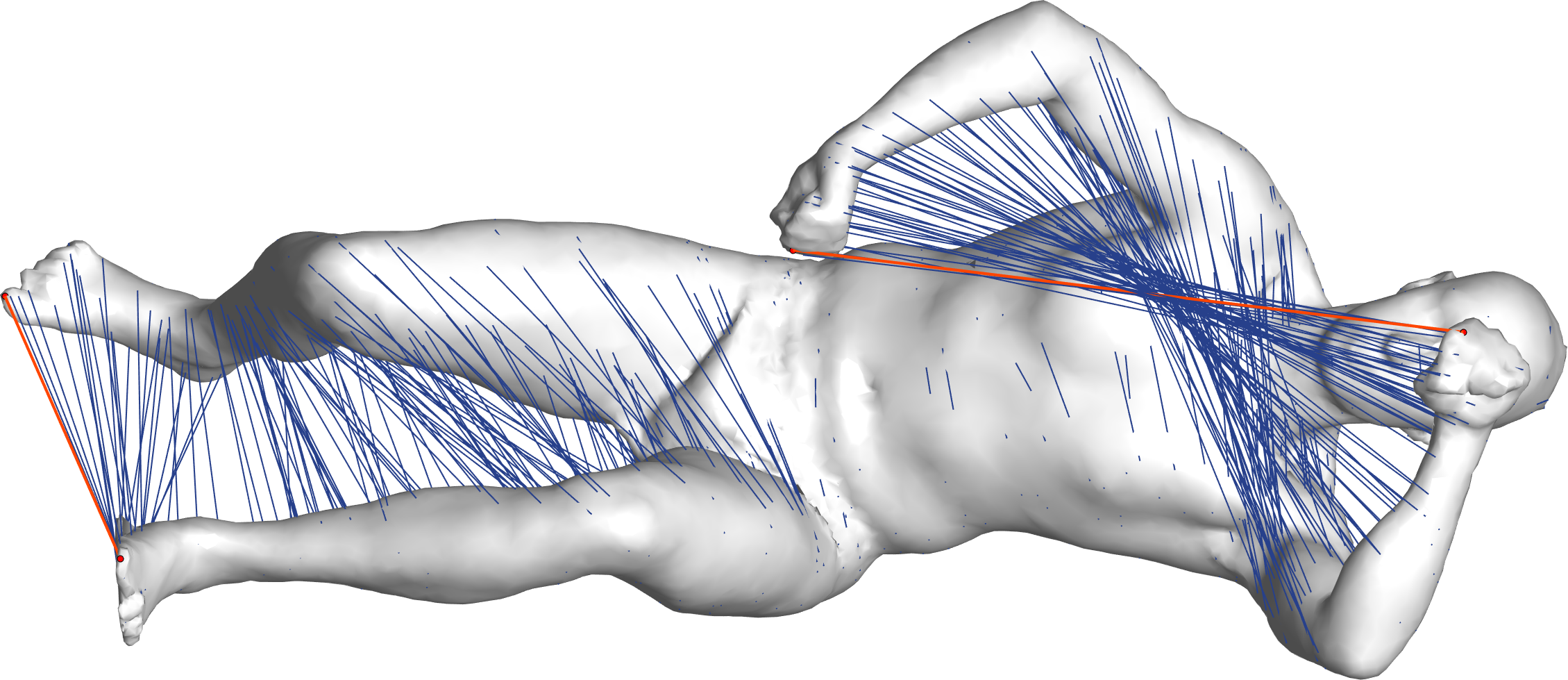,width=0.25\linewidth}
			\epsfig{figure=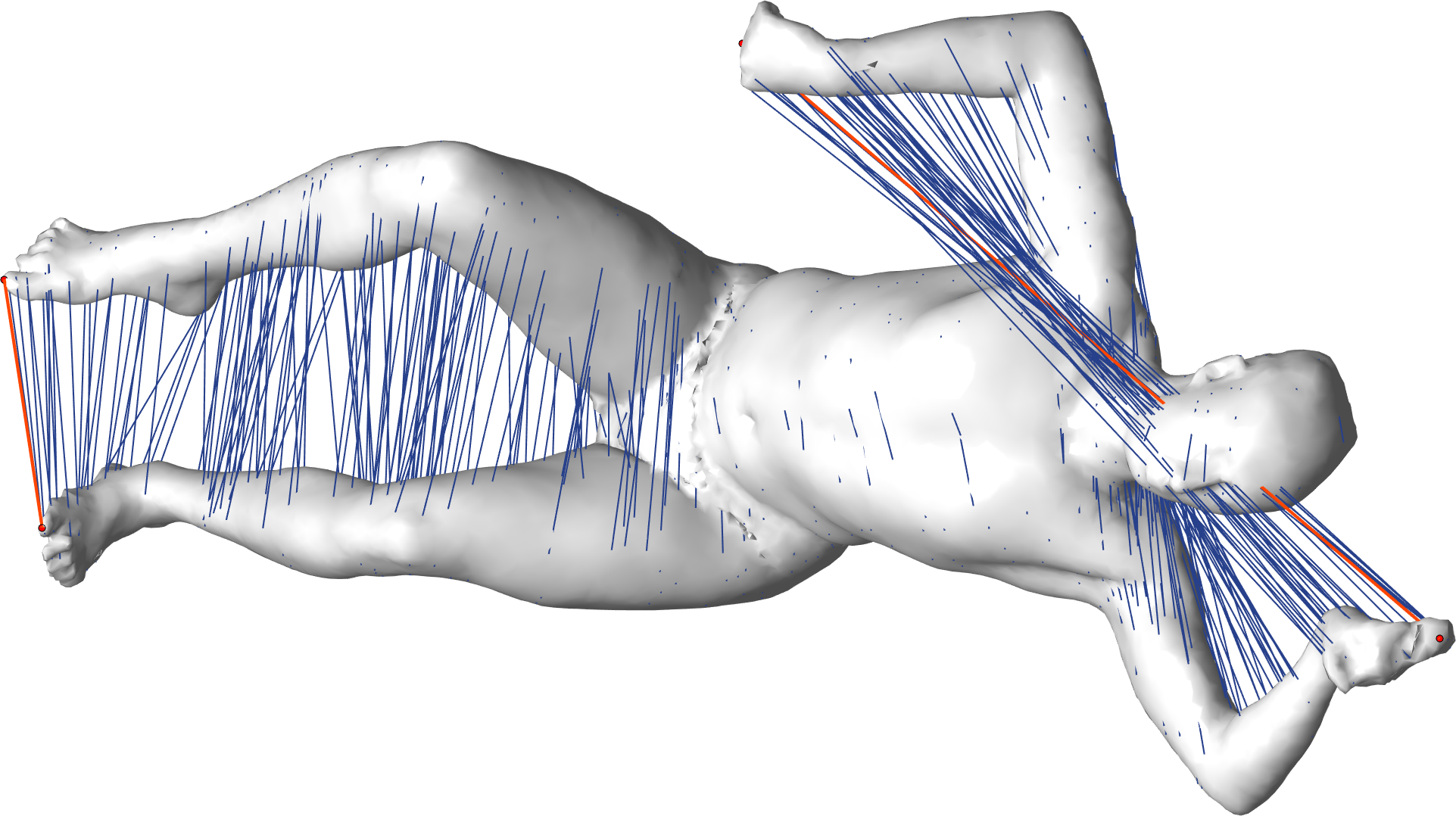,width=0.25\linewidth}
			\epsfig{figure=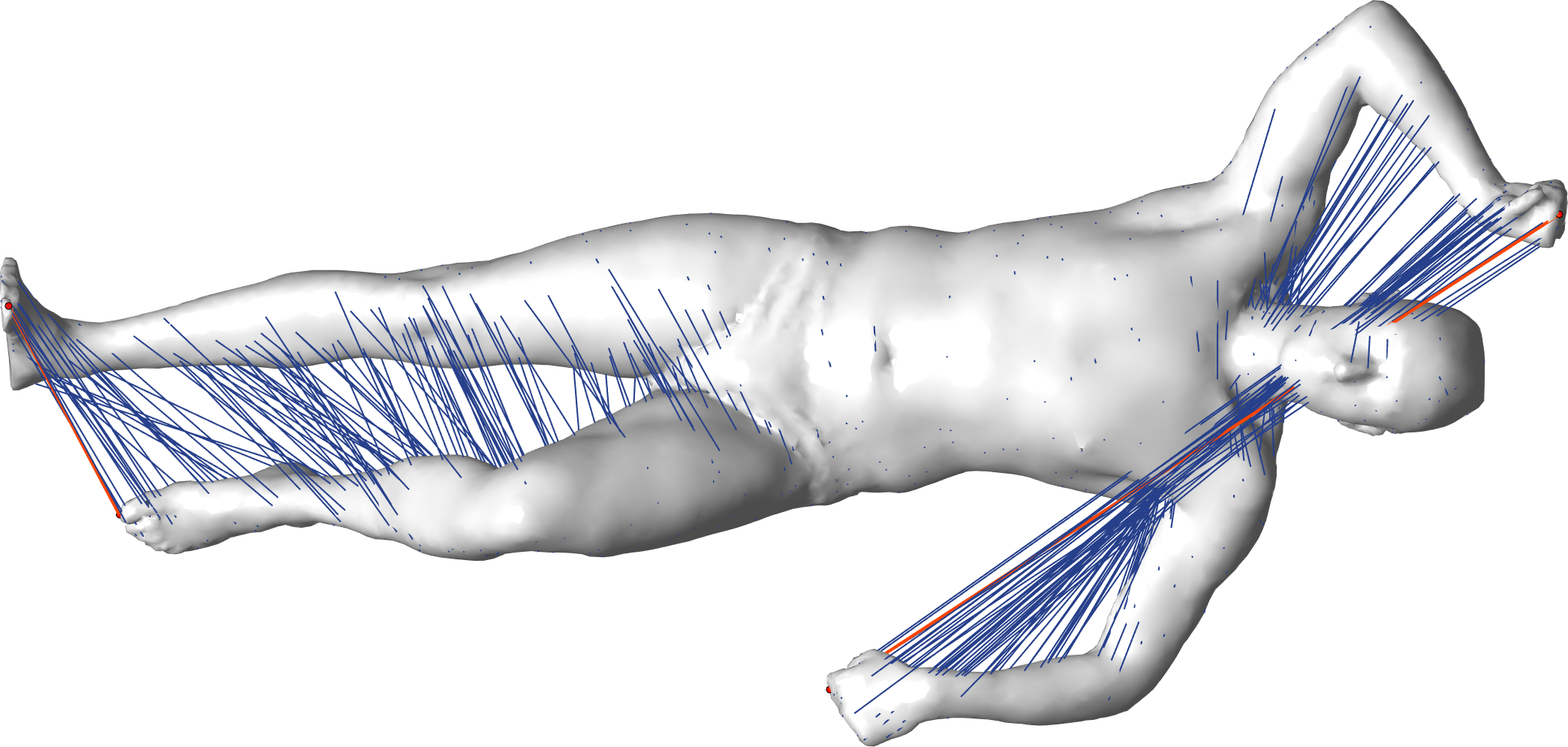,width=0.25\linewidth}
			\epsfig{figure=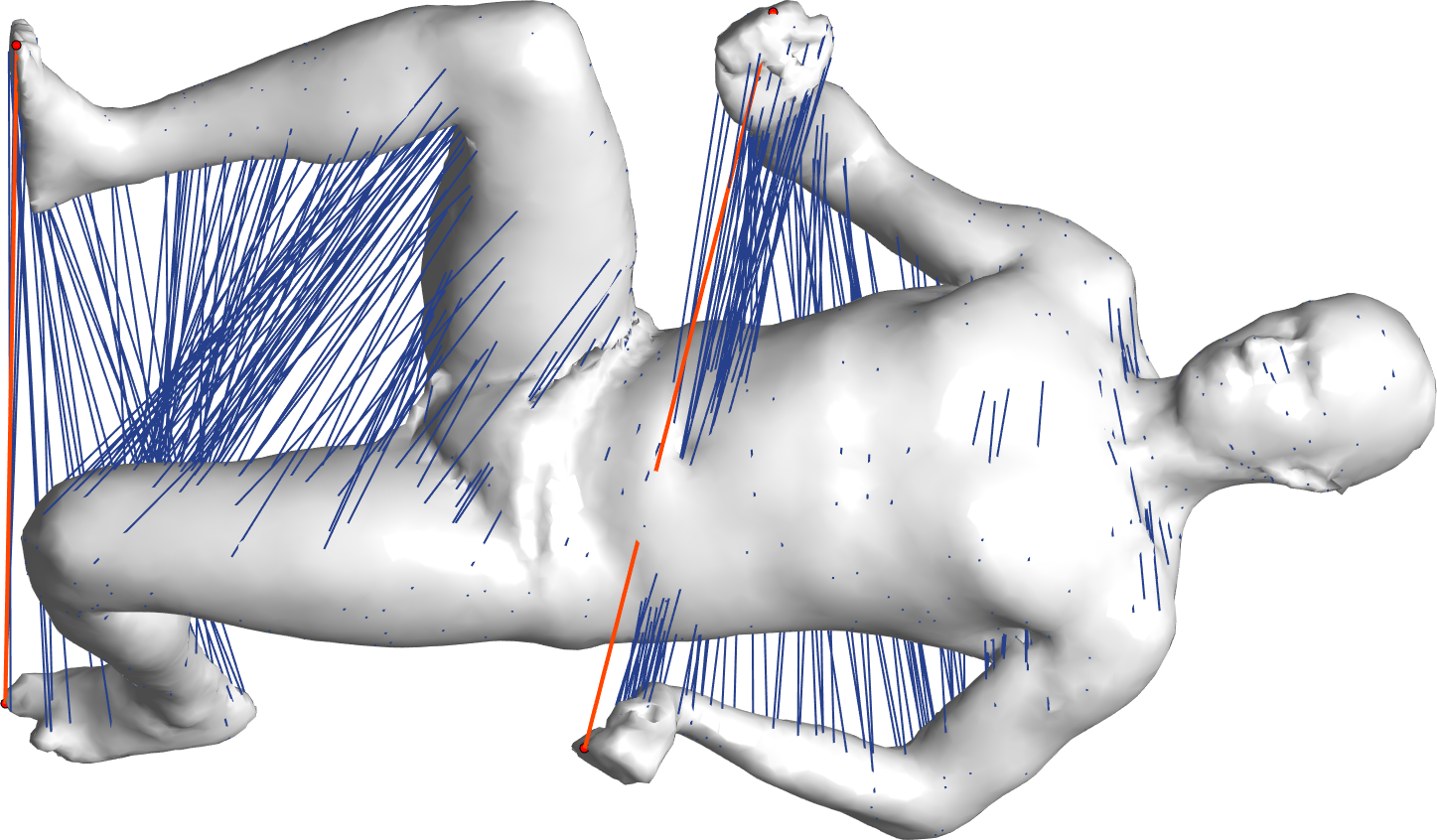,width=0.22\linewidth}
			\caption{Results of our approach on the TOSCA \cite{bronstein2008numerical}(first row) and the SCAPE \cite{anguelov2005scape} (second row) datasets. Detected correspondences (sparse) are shown in blue color. Correspondences in red color are the ones detected in Section \ref{subsec:3_4}.}
			\label{res_main}
		\end{center}
	\end{figure} 
%	\begin{figure}[t!]
%		\begin{center}
%			\epsfig{figure=gorilla08.png,width=0.18\linewidth}
%			\epsfig{figure=gorilla05.png,width=0.12\linewidth}
%			\epsfig{figure=dog01.png,width=0.14\linewidth}
%			\epsfig{figure=dog03.png,width=0.1\linewidth}
%			\epsfig{figure=centaur00.png,width=0.18\linewidth}
%			\epsfig{figure=centaur01.png,width=0.20\linewidth}
%			\epsfig{figure=mesh006.png,width=0.25\linewidth}
%			\epsfig{figure=mesh061.png,width=0.25\linewidth}
%			\epsfig{figure=mesh060.png,width=0.25\linewidth}
%			\epsfig{figure=mesh069.png,width=0.22\linewidth}
%			\caption{Results of our approach on the TOSCA \cite{bronstein2008numerical}(first row) and the SCAPE \cite{anguelov2005scape} (second row) datasets. Detected correspondences (sparse) are shown in blue color. Correspondences in red color are the ones detected in Section \ref{subsec:3_4}.}
%			\label{res_main}
%		\end{center}
%	\end{figure}
	$(x_j,x_{j^\prime}^\text{e})$ be the estimated correspondence, then the correspondence $(x_j,x_{j^\prime}^\text{e})$ is called true positive if the geodesic distance between the points $x_{j^\prime}^\text{g}$ and $x_{j^\prime}^\text{e}$ is less than $\sqrt{area(\mathcal{T})/20\pi}$ as used in \cite{kim2010mobius}. The correspondence rate is the fraction of true positive correspondences in the total estimated correspondences. \textbf{Mesh rate}: The mesh rate is the fraction of shapes for which the correspondence rate is more than $75\%$ in the total shapes as used in \cite{kim2010mobius}. \textbf{Time Complexity}: Total time required for computing symmetry for each shape in the given dataset. 
	\textbf{Datasets}. We evaluate our approach on the SCAPE \cite{anguelov2005scape} and TOSCA \cite{bronstein2008numerical} datasets. The SCAPE dataset contains 71 models. Each model in SCAPE dataset contains 12500 vertices and 24998 faces. The TOSCA dataset contains 80 models. On an average 20 ground truth intrinsically symmetric correspondences provided for each model in the datasets SCAPE and TOSCA. In the Fig.\ref{res_main}, we show a few results of the proposed approach on both the datasets.  We have only shown the sparsely detected correspondences for better visualization. \\ \textbf{Comparison Methods.} We compare the results of our approach on the datasets SCAPE and TOSCA with the four methods M{\"o}bius transformation voting (MT) \cite{kim2010mobius}, Blended Intrinsic Maps (BIM) \cite{kim2011blended}, Properly Constrained Orthogonal Functional Map (OFM) \cite{liu2015properly}, and Group Representation of Symmetries (GRS) \cite{wang2017group}.\\ \textbf{Discussions on the comparison.}  In Table \ref{tab:time}, we present the total time required for detecting the intrinsic symmetry in all the models of the TOSCA dataset for all the methods. We observe that our method is the fastest method on the TOSCA dataset. Our method takes around 6 seconds for each model whereas the method BIM takes around 270 seconds, the method OFM takes around 45 seconds, and the method GRS takes around 18 seconds. Our method takes 4.2 minutes to compute intrinsic symmetry in all the models of the SCAPE dataset. The possible reasons for our faster computation include finding the correspondence matrix using a closed form solution and determining the sign of eigenfunctions by computing the approximate shortest length geodesic curves between two intrinsically symmetric points. In Tables \ref{tab:res_scape}  and \ref{tab:res_tosca}, we present the correspondence rate (Corr rate) and the mesh rate for all the methods for all the models of the SCAPE and the TOSCA datasets, respectively.  The mesh rate for our method is equal to 100\% and the correspondence rate is equal to 97.8\% which is very close to the state-of-the-art correspondence rate 98\% of the method \cite{kim2011blended}. However, the average computation time for each mesh is around 270 seconds for the method \cite{kim2011blended}, whereas it is around 6 seconds for our method. We achieve the state-of-the-art performance on the SCAPE dataset. \begin{figure}[t!]
		\begin{center}
			\epsfig{figure=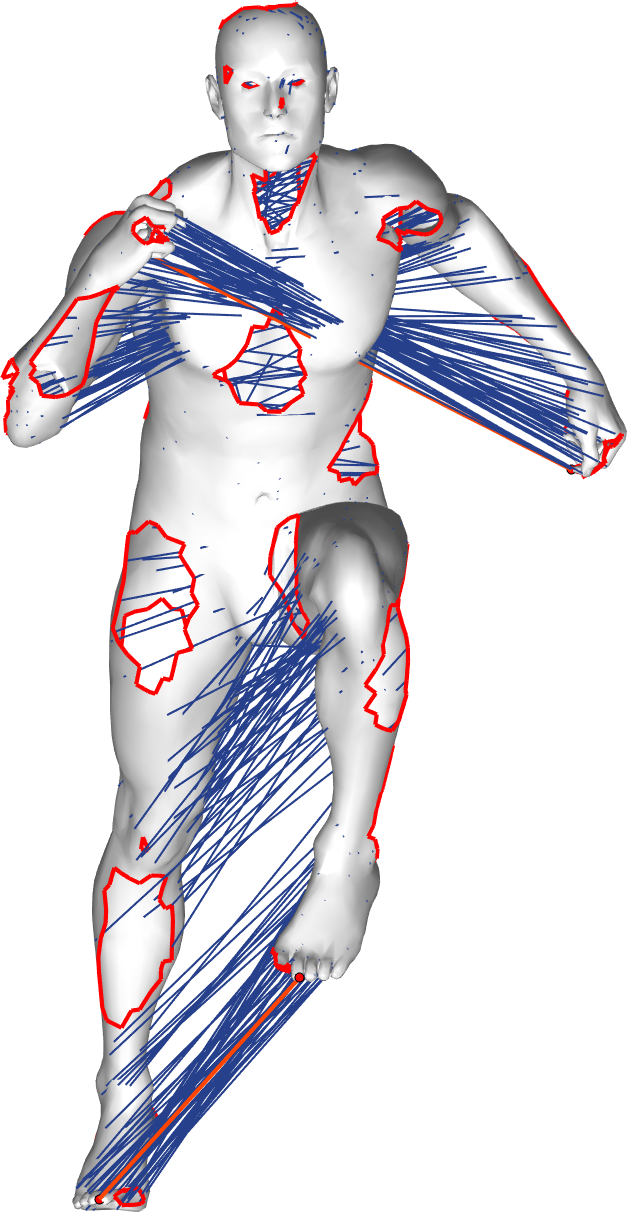,width=0.09\linewidth}
			\epsfig{figure=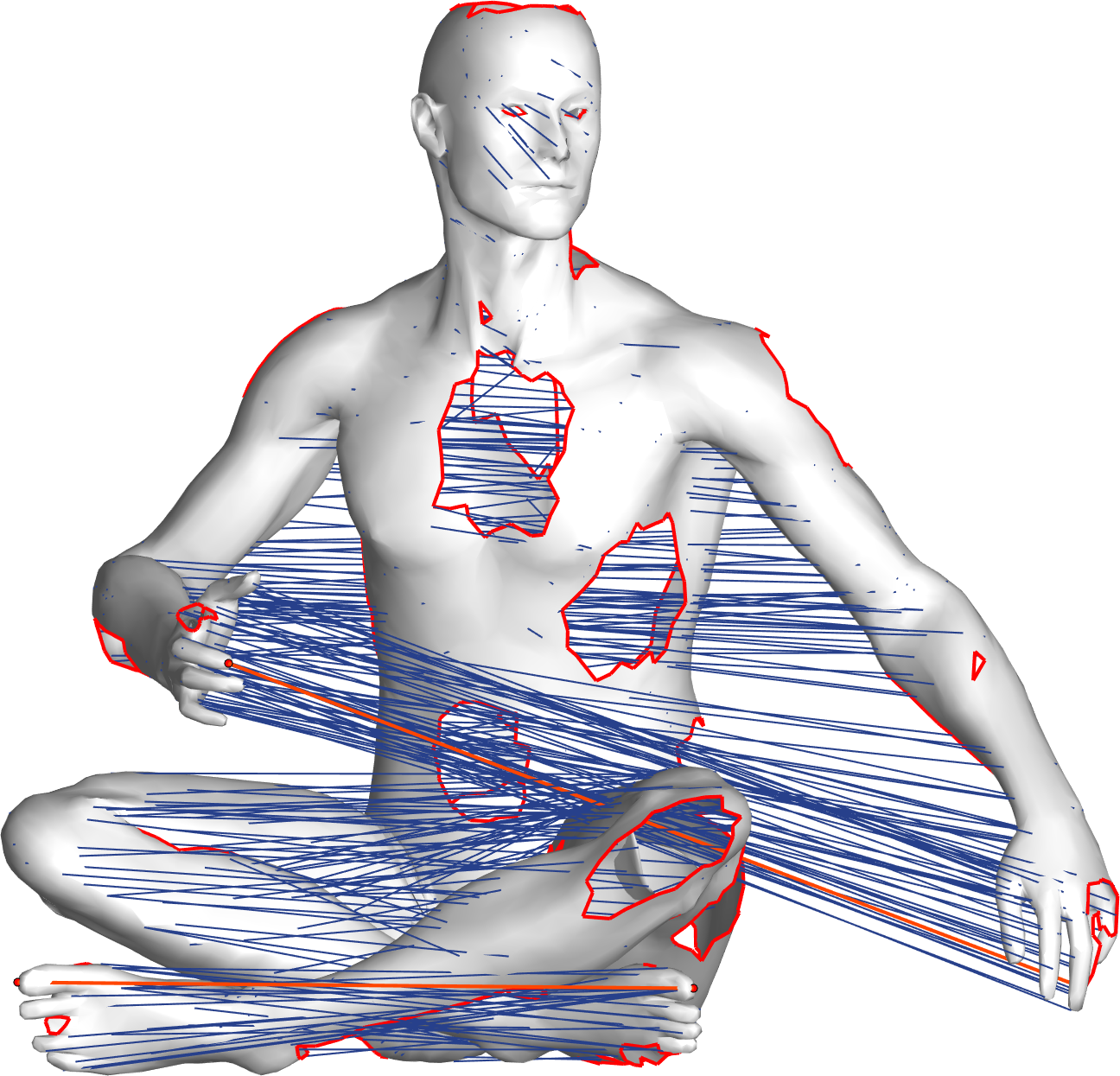,width=0.16\linewidth}
			\epsfig{figure=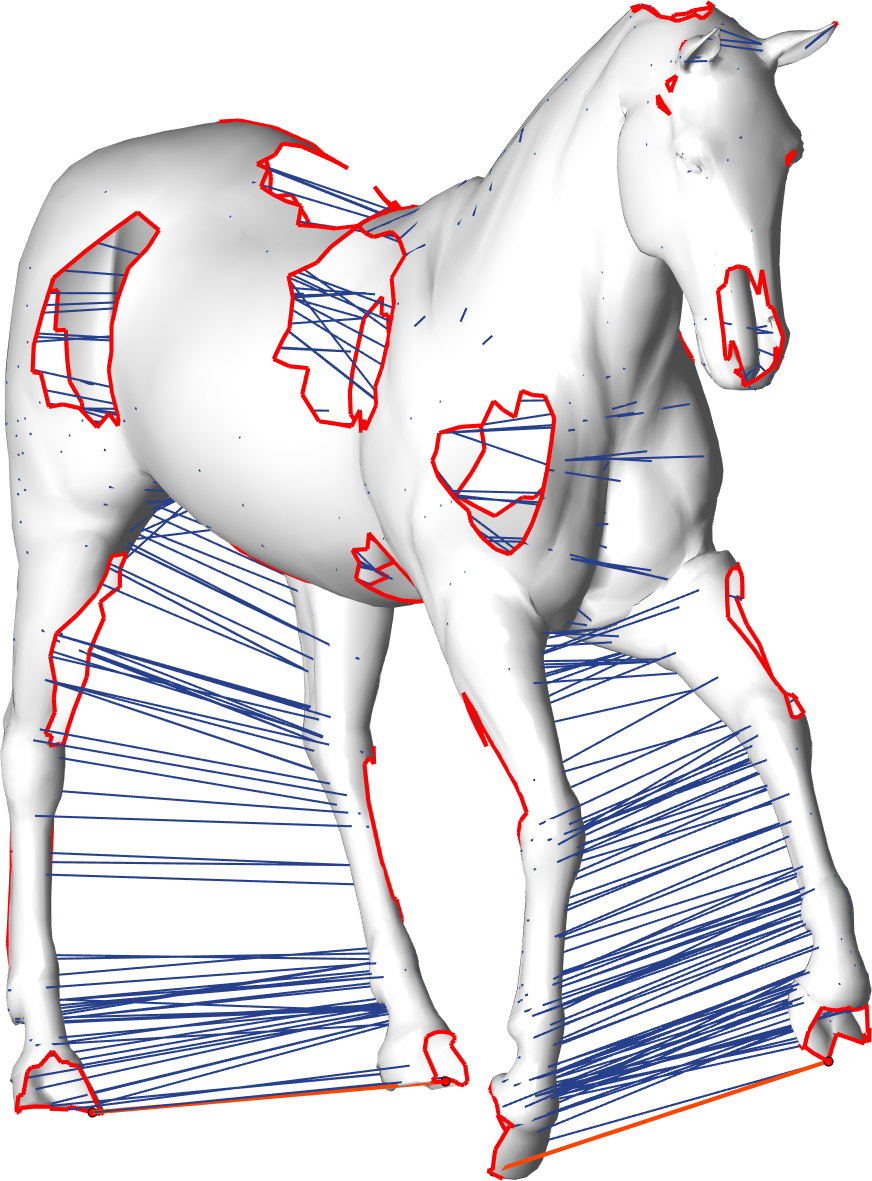,width=0.13\linewidth}
			\epsfig{figure=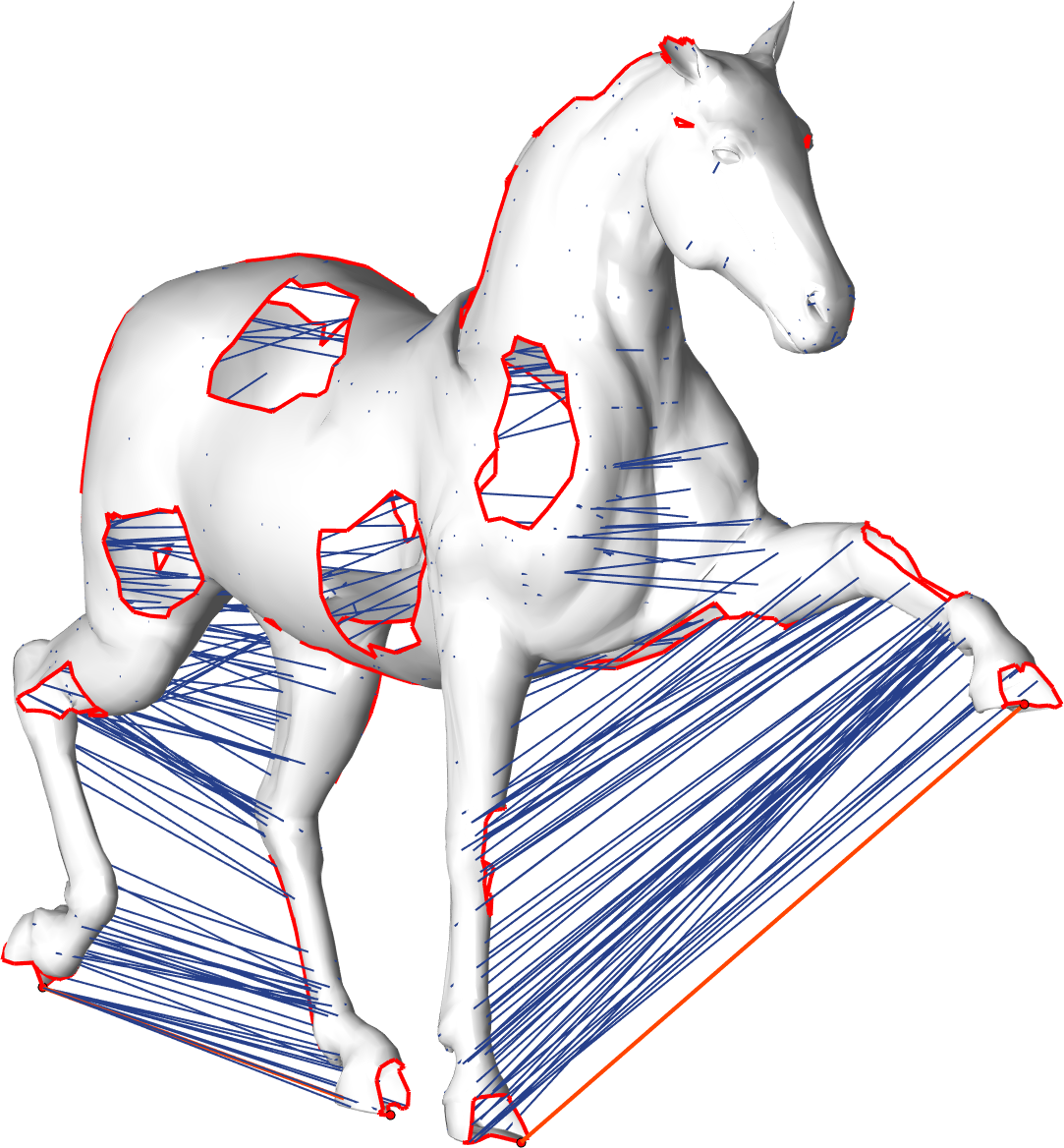,width=0.17\linewidth}
			\epsfig{figure=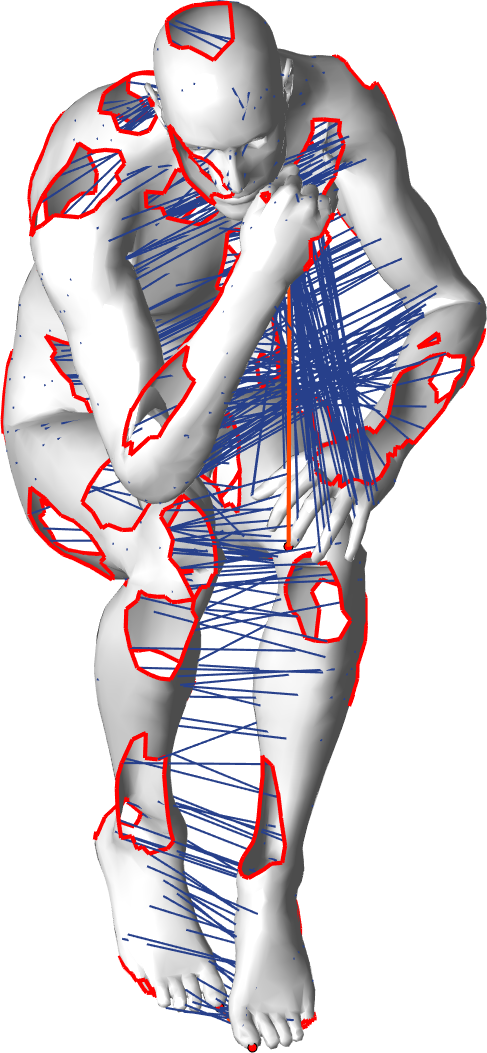,width=0.08\linewidth}            
			\epsfig{figure=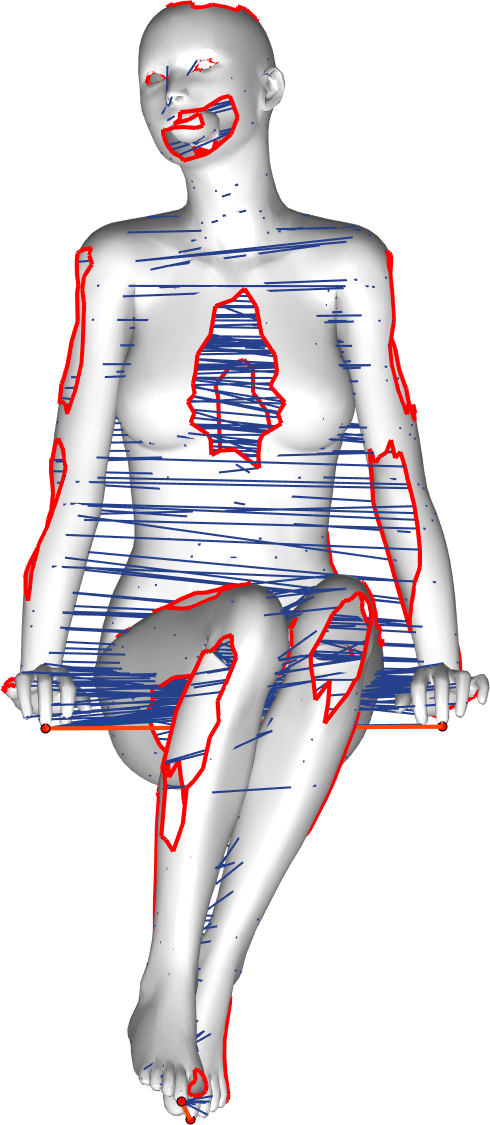,width=0.08\linewidth}
			\epsfig{figure=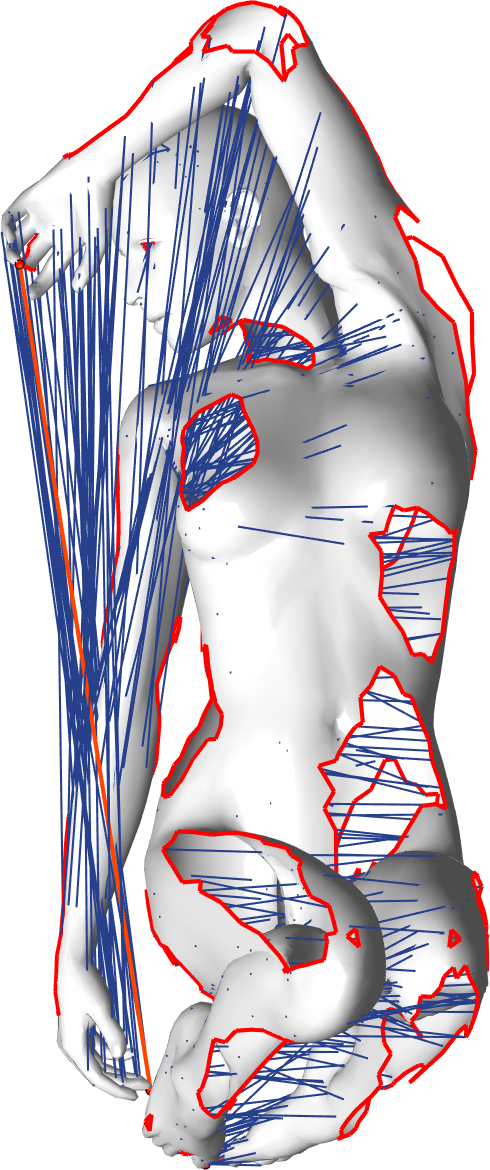,width=0.08\linewidth}
			\caption{Partial intrinsic symmetry detection results on the dataset SHREC16 \cite{cosmo2016shrec}.}
			\label{res_par}
		\end{center}
	\end{figure}\\ \textbf{Effect of holes}.  In Fig. \ref{res_par}, we show the detected intrinsic symmetry in the partial model from the SHREC16 \cite{cosmo2016shrec} dataset. Here, the partial shape is obtained by making holes in the original shape such that it contains 90\% area of the main shape. We observe that our method is invariant to significant holes.   
	\section{Conclusions}
	We have presented a fast and an accurate algorithm for detecting intrinsic symmetry in triangle meshes. We showed that the  functional correspondence matrix is diagonal and a diagonal entry is $+1\; (-1)$ if the corresponding eigenfunction is even (odd). We showed that the restriction of an even (odd) eigenfunction on the shortest length geodesic between any two intrinsically symmetric points also is an even (odd) function. This result has helped us to derive a closed form solution to find diagonal entries of this matrix. We achieved state-of-the-art performance on the SCAPE dataset and second best on the TOSCA dataset. We achieved the best time complexity. Furthermore, our approach is invariant to the ordering of eigenfunctions and robust to the presence of holes in the input mesh.
	Our method is limited to the intrinsic reflective symmetry. It can not find the other types of symmetries such as rotational symmetry. We would like to extend our approach to more general symmetries. Our approach may fail to detect intrinsic symmetry in non-connected manifolds. As future work, we would like to extend the functional map to detect intrinsic symmetries in non-connected manifolds.\\
	\textbf{Acknowledgment:} R. Nagar was supported by the TCS research Scholarship.
	\section*{Appendix}
	\textbf{Computing Riemannian gradient and Hessians of the Problem defined in Eq. (4) of the paper.}\\
	Consider the optimization problem defined in Eq. (3) of the main paper. Which is
	\begin{eqnarray}
	\nonumber&\underset{\mathbf{R}}{\min} \text{ off}(\mathbf{R}^\top\mathbf{D}\mathbf{R})+\|\mathbf{R}^\top\mathbf{D}\mathbf{R}-\mathbf{D}\|_\text{F}^2\\ \text{subject to}&\mathbf{R}^\top\mathbf{\Phi}^\top\mathbf{F}=\mathbf{C}\mathbf{R}^\top\mathbf{\Phi}^\top\mathbf{G},\;
	\mathbf{R}^\top\mathbf{R}=\mathbf{I},\text{det}(\mathbf{R})=+1,\mathbf{R}\in\mathbb{R}^{k\times k}.
	\label{eq:31}
	\end{eqnarray}    
	Here, $\mathbf{R}^\top\mathbf{R}=\mathbf{I}$ follows from the fact that the transformed basis $\mathbf{\Phi} \mathbf{R}$ is an orthogonal basis.  We have $\mathbf{R}^\top\mathbf{\Phi}^\top\mathbf{A}\mathbf{\Phi R}=\mathbf{I}\Rightarrow \mathbf{R}^\top\mathbf{R}=\mathbf{I}$ since $\mathbf{ \Phi}^\top\mathbf{A}\mathbf{\Phi}=\mathbf{I}$. We observe that the set  $\{\mathbf{R}\in\mathbb{R}^{k\times k}:\mathbf{R}^\top\mathbf{R}=\mathbf{I},\text{det}(\mathbf{R})=+1\}$ is the special orthogonal group $\mathcal{SO}(k)$.  Therefore, we solve the optimization problem  
	\begin{eqnarray}
	\underset{\mathbf{R}\in \mathcal{SO}(k)}{\min}\text{ off}(\mathbf{R}^\top\mathbf{D}\mathbf{R})+\|\mathbf{R}^\top\mathbf{D}\mathbf{R}-\mathbf{D}\|_\text{F}^2+\mu\|\mathbf{R}^\top\mathbf{\Phi}^\top\mathbf{F}-\mathbf{C}\mathbf{R}^\top\mathbf{\Phi}^\top\mathbf{G}\|_\text{F}^2
	\label{eq:41}
	\end{eqnarray}   
	Here, we choose $\mu=1$. Now, let $\bar{f}=\bar{f}_1+\bar{f}_2+\bar{f}_3$. Where, $\bar{f}:\mathcal{SO}(k)\rightarrow \mathbb{R}$, and $\bar{f}_i:\mathcal{SO}(k)\rightarrow \mathbb{R}$, $i=1,2,3$, and are defined as  $\bar{f}_1(\mathbf{R})=\text{off}(\mathbf{R}^\top\mathbf{D}\mathbf{R})$, $\bar{f}_2(\mathbf{R})=\|\mathbf{R}^\top\mathbf{D}\mathbf{R}-\mathbf{D}\|_\text{F}^2$, and  $\bar{f}_3(\mathbf{R})=\|\mathbf{R}^\top\mathbf{\Phi}^\top\mathbf{F}-\mathbf{C}\mathbf{R}^\top\mathbf{\Phi}^\top\mathbf{G}\|_\text{F}^2$.
	
	Let, $f:\mathbb{R}^{k\times k}\rightarrow\mathbb{R}$ be a scalar function on $\mathbb{R}^{k\times k}$ and defined as $f(\mathbf{R})=\text{ off}(\mathbf{R}^\top\mathbf{D}\mathbf{R})+\|\mathbf{R}^\top\mathbf{D}\mathbf{R}-\mathbf{D}\|_\text{F}^2+\|\mathbf{R}^\top\mathbf{\Phi}^\top\mathbf{F}-\mathbf{C}\mathbf{R}^\top\mathbf{\Phi}^\top\mathbf{G}\|_\text{F}^2$. Let $\nabla_\mathbf{R}f$ be its gradient (also called classical gradient).  Since the function  $\bar{f}$ is the restriction function of the function $f$ on $\mathcal{SO}(k)$, i.e., $\bar{f}:\mathcal{SO}(k)\rightarrow\mathbb{R}$ and $\bar{f}(\mathbf{R})=f(\mathbf{R})\forall\mathbf{R}\in\mathcal{SO}(k)$, the Rimannian gradient of the function $\bar{f}$ is defined as $\text{grad}_\mathbf{R} \bar{f}=\mathbb{P}(\nabla_\mathbf{R} f)=\frac{\mathbf{R}(\mathbf{R}^\top\nabla_\mathbf{R} f-(\nabla_\mathbf{R} f)^\top\mathbf{\mathbf{R}})}{2}$ \cite{absil2009optimization}. Therefore, in order to find the Riemannian gradient of the function $\bar{f}$, first we need to find the classical gradient $\nabla_\mathbf{R} f$ of the function $f$.  
	
	Now let $f=f_1+f_2+f_3$. Where, $f_i:\mathbb{R}^{k\times k}\rightarrow \mathbb{R}$, $i=1,2,3$, and defined as  $f_1(\mathbf{R})=\text{off}(\mathbf{R}^\top\mathbf{D}\mathbf{R})$, $f_2(\mathbf{R})=\|\mathbf{R}^\top\mathbf{D}\mathbf{R}-\mathbf{D}\|_\text{F}^2$, and  $f_3(\mathbf{R})=\|\mathbf{R}^\top\bar{\mathbf{F}}-\mathbf{C}\mathbf{R}^\top\bar{\mathbf{G}}\|_\text{F}^2$. Where, $\bar{\mathbf{F}}=\mathbf{\Phi}^\top\mathbf{F}$ and $\bar{\mathbf{G}}=\mathbf{\Phi}^\top\mathbf{G}$. We follow \cite{kovnatsky2013coupled} to find the gradients of the  function $f_1$ and $f_2$ which are defined as $\nabla_{\mathbf{R}}f_1=4(\mathbf{DDR}-\mathbf{O}\odot\mathbf{RD})$ and $\nabla_{\mathbf{R}}f_2=4(\mathbf{DDR}-\mathbf{DRD})$. Here, the matrix $\mathbf{O}=\mathbf{1}^\top\otimes \begin{bmatrix} \lambda^\prime_1&\lambda^\prime_2&\ldots&\lambda^\prime_k\end{bmatrix}^\top$,  $\mathbf{1}$ is a vector of size $k\times 1$ with all elements equal to 1, and $\lambda^\prime_1,\lambda^\prime_2,\ldots,\lambda^\prime_k$ are the diagonal entries of the matrix $\mathbf{R^\top DR}$. Now, we find the gradient of the function $f_3$ as follows. 
	\begin{eqnarray}
	\nonumber f_3(\mathbf{R})&=&\text{trace}((\mathbf{R}^\top\bar{\mathbf{F}}-\mathbf{C}\mathbf{R}^\top\bar{\mathbf{G}})^\top(\mathbf{R}^\top\bar{\mathbf{F}}-\mathbf{C}\mathbf{R}^\top\bar{\mathbf{G}}))\\
	\nonumber&=&\text{trace}((\bar{\mathbf{F}}^\top\mathbf{R}-\bar{\mathbf{G}}^\top\mathbf{R}\mathbf{C}^\top)(\mathbf{R}^\top\bar{\mathbf{F}}-\mathbf{C}\mathbf{R}^\top\bar{\mathbf{G}}))\\
	\nonumber&=&\text{trace}(\bar{\mathbf{F}}^\top\mathbf{R}\mathbf{R}^\top\bar{\mathbf{F}}-\bar{\mathbf{F}}^\top\mathbf{R}\mathbf{C}\mathbf{R}^\top\bar{\mathbf{G}}-\bar{\mathbf{G}}^\top\mathbf{R}\mathbf{C}^\top\mathbf{R}^\top\bar{\mathbf{F}}+\bar{\mathbf{G}}^\top\mathbf{R}\mathbf{C}^\top\mathbf{C}\mathbf{R}^\top\bar{\mathbf{G}})\\
	\nonumber&=&\text{trace}(\bar{\mathbf{F}}^\top\bar{\mathbf{F}}-\bar{\mathbf{F}}^\top\mathbf{R}\mathbf{C}\mathbf{R}^\top\bar{\mathbf{G}}-\bar{\mathbf{G}}^\top\mathbf{R}\mathbf{C}\mathbf{R}^\top\bar{\mathbf{F}}+\bar{\mathbf{G}}^\top\bar{\mathbf{G}})\\
	\nonumber&=&\text{trace}(\bar{\mathbf{F}}^\top\bar{\mathbf{F}})-\text{trace}(\bar{\mathbf{F}}^\top\mathbf{R}\mathbf{C}\mathbf{R}^\top\bar{\mathbf{G}})-\text{trace}(\bar{\mathbf{G}}^\top\mathbf{R}\mathbf{C}\mathbf{R}^\top\bar{\mathbf{F}})+\text{trace}(\bar{\mathbf{G}}^\top\bar{\mathbf{G}})\\
	\nonumber&=&\text{trace}(\bar{\mathbf{F}}^\top\bar{\mathbf{F}})-2\text{trace}(\bar{\mathbf{G}}^\top\mathbf{R}\mathbf{C}\mathbf{R}^\top\bar{\mathbf{F}})+\text{trace}(\bar{\mathbf{G}}^\top\bar{\mathbf{G}})
	\end{eqnarray}

	Now we have that $\nabla_\mathbf{R}(\text{trace}(\bar{\mathbf{F}}^\top\bar{\mathbf{F}}))=\mathbf{0}$ and $\nabla_\mathbf{R}(\text{trace}(\bar{\mathbf{G}}^\top\bar{\mathbf{G}}))=\mathbf{0}$. We use the definitions defined in \cite{petersen2008matrix} to find that $\nabla_\mathbf{R}(\text{trace}(\bar{\mathbf{G}}^\top\mathbf{R}\mathbf{C}\mathbf{R}^\top\bar{\mathbf{F}}))=(\bar{\mathbf{F}}\bar{\mathbf{G}}^\top+\bar{\mathbf{G}}\bar{\mathbf{F}}^\top)\mathbf{R}\mathbf{C}$. Therefore, $
	\nabla_\mathbf{R}f_3=-2(\bar{\mathbf{F}}\bar{\mathbf{G}}^\top+\bar{\mathbf{G}}\bar{\mathbf{F}}^\top)\mathbf{R}\mathbf{C}.
	$
	Therefore, we have
	$$
	\nabla_\mathbf{R} f=4(\mathbf{DDR}-\mathbf{O}\odot\mathbf{RD})+4(\mathbf{DDR}-\mathbf{DRD})-2(\bar{\mathbf{F}}\bar{\mathbf{G}}^\top+\bar{\mathbf{G}}\bar{\mathbf{F}}^\top)\mathbf{R}\mathbf{C}.
	$$
	Since $\text{grad}_\mathbf{R} f=\mathbb{P}(\nabla \bar{f})=\frac{\mathbf{R}(\mathbf{R}^\top\nabla f-(\nabla f)^\top\mathbf{\mathbf{R}})}{2}$, therefore we have that 
	
	$$
	\text{grad}_\mathbf{R}\bar{f}=2(\mathbf{R}(\mathbf{O}\odot\mathbf{RD})^\top\mathbf{R}-\mathbf{O}\odot\mathbf{RD})+2(\mathbf{RDR}^\top\mathbf{DR}-\mathbf{DRD})-\mathbf{SRC}+\mathbf{RCR^\top SR}
	$$
	
	Here, $\mathbf{S}=\mathbf{\Phi^\top FG^\top\Phi}+\mathbf{\Phi^\top FG^\top\Phi}$.  The analytical Reimannian Hessian contains many matrix multiplications (as shown below). Therefore, we use the numerical approach option in the \texttt{manopt} toolbox \cite{boumal2014manopt} to find the  Reimannian Hessian. 
	
	The Riemannian Hessian of the function $\bar{f}$ at $\mathbf{R}$ in the direction $\mathbf{R\Omega}$ is defined as $\text{Hess}_{\mathbf{R}}(\bar{f})[\mathbf{R\Omega}]=\mathbb{P}(\mathbb{D}(\nabla_\mathbf{R} f)[\mathbf{R\Omega}])$ \cite{absil2009optimization}. Where, $\mathbb{D}(\text{grad}_\mathbf{R} f)[\mathbf{R\Omega}]$ is the classical derivative of $\text{grad}_\mathbf{R} f$ in the direction $\mathbf{R\Omega}$ which is an element of the tangent space of the $\mathcal{SO}(k)$ at the point $\mathbf{R}$ and defined as 
	$\mathbb{D}(\text{grad}_\mathbf{R} \bar{f})[\mathbf{R\Omega}]=\frac{d}{dt}\text{grad}_\mathbf{R} \bar{f}(\mathbf{R}+t\mathbf{R\Omega})|_{t=0}$. Now
	\begin{eqnarray}
	\nonumber\mathbb{D}(\text{grad}_\mathbf{R} \bar{f})[\mathbf{R\Omega}]&=&\frac{d}{dt}\text{grad}_\mathbf{R} \bar{f}_1(\mathbf{R}+t\mathbf{R\Omega})|_{t=0}+\frac{d}{dt}\text{grad}_\mathbf{R} \bar{f}_2(\mathbf{R}+t\mathbf{R\Omega})|_{t=0}\\
	&&+\frac{d}{dt}\text{grad}_\mathbf{R} \bar{f}_3(\mathbf{R}+t\mathbf{R\Omega})|_{t=0}
	\end{eqnarray}
	\begin{eqnarray}
	\nonumber\mathbb{D}(\text{grad}_\mathbf{R} \bar{f}_1)[\mathbf{R\Omega}]&=&\frac{d}{dt}\text{grad}_\mathbf{R}\bar{f}_1(\mathbf{R}+t\mathbf{R\Omega})|_{t=0}\\
	\nonumber&=&\frac{d}{dt}\bigg(2((\mathbf{R}+t\mathbf{R\Omega})(\mathbf{O}\odot(\mathbf{R}+t\mathbf{R\Omega})\mathbf{D})^\top(\mathbf{R}+t\mathbf{R\Omega})\\
	\nonumber&&-\mathbf{O}\odot(\mathbf{R+R\Omega})\mathbf{D})\bigg)\bigg|_{t=0}\\
	\nonumber&=&2(\mathbf{R(O\odot RD)^\top R\Omega}+\mathbf{R(O\odot R\Omega D)^\top R}\\
	&&+\mathbf{R\Omega (O\odot RD)^\top R}-\mathbf{O\odot R\Omega D})
	\end{eqnarray}
	\begin{eqnarray}
	\nonumber\mathbb{D}(\text{grad}_\mathbf{R} \bar{f}_2)[\mathbf{R\Omega}]&=&\frac{d}{dt}\text{grad}_\mathbf{R}\bar{f}_2(\mathbf{R}+t\mathbf{R\Omega})|_{t=0}\\
	\nonumber&=&\frac{d}{dt}\bigg(2((\mathbf{R}+t\mathbf{R\Omega})D(\mathbf{R}+t\mathbf{R\Omega})^\top\mathbf{D}(\mathbf{R}+t\mathbf{R\Omega})\\
	\nonumber&&-\mathbf{D}(\mathbf{R}+t\mathbf{R\Omega})\mathbf{D})\bigg)\bigg|_{t=0}\\
	\nonumber&=&2(\mathbf{RDR^\top DR\Omega}+\mathbf{RD\Omega^\top R^\top DR}\\
	&&+\mathbf{R\Omega DR^\top DR}-\mathbf{DR\Omega D})
	\end{eqnarray}
	
	\begin{eqnarray}
	\nonumber\mathbb{D}(\text{grad}_\mathbf{R} \bar{f}_3)[\mathbf{R\Omega}]&=&\frac{d}{dt}\text{grad}_\mathbf{R}\bar{f}_3(\mathbf{R}+t\mathbf{R\Omega})|_{t=0}\\
	\nonumber&=&\frac{d}{dt}\bigg(2((\mathbf{R}+t\mathbf{R\Omega})C(\mathbf{R}+t\mathbf{R\Omega})^\top\mathbf{S}(\mathbf{R}+t\mathbf{R\Omega})\\
	\nonumber&&-\mathbf{S}(\mathbf{R}+t\mathbf{R\Omega})\mathbf{C})\bigg)\bigg|_{t=0}\\
	\nonumber&=&\mathbf{RCR^\top SR\Omega}+\mathbf{RC\Omega^\top R^\top SR}\\
	&&+\mathbf{R\Omega CR^\top SR}-\mathbf{SR\Omega C}
	\end{eqnarray}
	Therefore,
	\begin{eqnarray}
	\nonumber\mathbb{D}(\text{grad}_\mathbf{R} \bar{f})[\mathbf{R\Omega}]&=&2(\mathbf{R(O\odot RD)^\top R\Omega}+\mathbf{R(O\odot R\Omega D)^\top R}\\
	\nonumber&&+\mathbf{R\Omega (O\odot RD)^\top R}-\mathbf{O\odot R\Omega D})\\
	\nonumber&&+2(\mathbf{RDR^\top DR\Omega}+\mathbf{RD\Omega^\top R^\top DR}\\
	\nonumber&&+\mathbf{R\Omega DR^\top DR}-\mathbf{DR\Omega D})\\
	\nonumber&&+\mathbf{RCR^\top SR\Omega}+\mathbf{RC\Omega^\top R^\top SR}\\
	&&+\mathbf{R\Omega CR^\top SR}-\mathbf{SR\Omega C}
	\end{eqnarray}
	Now, the Riemannian Hessian  $\text{Hess}_{\mathbf{R}}(\bar{f})[\mathbf{R\Omega}]=\mathbb{P}(\mathbb{D}(\text{grad}_\mathbf{R} \bar{f})[\mathbf{R\Omega}])$ which contains many matrix multiplications.

	\bibliographystyle{splncs04}

\begin{thebibliography}{10}
		\providecommand{\url}[1]{\texttt{#1}}
		\providecommand{\urlprefix}{URL }
		\providecommand{\doi}[1]{https://doi.org/#1}
		
		\bibitem{absil2007trust}
		Absil, P.A., Baker, C.G., Gallivan, K.A.: Trust-region methods on riemannian
		manifolds. Foundations of Computational Mathematics  \textbf{7}(3),  303--330
		(2007)
		
		\bibitem{absil2009optimization}
		Absil, P.A., Mahony, R., Sepulchre, R.: Optimization algorithms on matrix
		manifolds. Princeton University Press (2009)
		
		\bibitem{anguelov2005scape}
		Anguelov, D., Srinivasan, P., Koller, D., Thrun, S., Rodgers, J., Davis, J.:
		Scape: shape completion and animation of people. In: ACM transactions on
		graphics (TOG). vol.~24, pp. 408--416. ACM (2005)
		
		\bibitem{arya1998optimal}
		Arya, S., Mount, D.M., Netanyahu, N.S., Silverman, R., Wu, A.Y.: An optimal
		algorithm for approximate nearest neighbor searching fixed dimensions.
		Journal of the ACM (JACM)  \textbf{45}(6),  891--923 (1998)
		
		\bibitem{berner2009generalized}
		Berner, A., Bokeloh, M., Wand, M., Schilling, A., Seidel, H.P.: Generalized
		intrinsic symmetry detection  (2009)
		
		\bibitem{berner2011shape}
		Berner, A., Wand, M., Mitra, N.J., Mewes, D., Seidel, H.P.: Shape analysis with
		subspace symmetries. In: Computer Graphics Forum. vol.~30, pp. 277--286.
		Wiley Online Library (2011)
		
		\bibitem{boumal2014manopt}
		Boumal, N., Mishra, B., Absil, P.A., Sepulchre, R.: Manopt, a matlab toolbox
		for optimization on manifolds. The Journal of Machine Learning Research
		\textbf{15}(1),  1455--1459 (2014)
		
		\bibitem{bronstein2008numerical}
		Bronstein, A.M., Bronstein, M.M., Kimmel, R.: Numerical geometry of non-rigid
		shapes. Springer Science \& Business Media (2008)
		
		\bibitem{cosmo2016shrec}
		Cosmo, L., Rodol{\`a}, E., Bronstein, M., Torsello, A., Cremers, D.,
		Sahillioglu, Y.: Shrec’16: Partial matching of deformable shapes. Proc.
		3DOR  \textbf{2}(9), ~12 (2016)
		
		\bibitem{dessein2017symmetry}
		Dessein, A., Smith, W.A.P., Wilson, R.C., Hancock, E.R.: Symmetry-aware mesh
		segmentation into uniform overlapping patches. In: Computer Graphics Forum.
		vol.~36, pp. 95--107. Wiley Online Library (2017)
		
		\bibitem{gallier2012notes}
		Gallier, J., Quaintance, J.: Notes on differential geometry and Lie groups.
		University of Pennsylvania (2018)
		
		\bibitem{ghosh2010closed}
		Ghosh, D., Amenta, N., Kazhdan, M.: Closed-form blending of local symmetries.
		In: Computer Graphics Forum. vol.~29, pp. 1681--1688. Wiley Online Library
		(2010)
		
		\bibitem{jiang2013skeleton}
		Jiang, W., Xu, K., Cheng, Z.Q., Zhang, H.: Skeleton-based intrinsic symmetry
		detection on point clouds. Graphical Models  \textbf{75}(4),  177--188 (2013)
		
		\bibitem{kazhdan2004symmetry}
		Kazhdan, M., Funkhouser, T., Rusinkiewicz, S.: Symmetry descriptors and 3d
		shape matching. In: Proceedings of the 2004 Eurographics/ACM SIGGRAPH
		symposium on Geometry processing. pp. 115--123. ACM (2004)
		
		\bibitem{kerber2013scalable}
		Kerber, J., Bokeloh, M., Wand, M., Seidel, H.P.: Scalable symmetry detection
		for urban scenes. In: Computer Graphics Forum. vol.~32, pp. 3--15. Wiley
		Online Library (2013)
		
		\bibitem{kim2010mobius}
		Kim, V.G., Lipman, Y., Chen, X., Funkhouser, T.: M{\"o}bius transformations for
		global intrinsic symmetry analysis. In: Computer Graphics Forum. vol.~29, pp.
		1689--1700. Wiley Online Library (2010)
		
		\bibitem{kim2011blended}
		Kim, V.G., Lipman, Y., Funkhouser, T.: Blended intrinsic maps. In: ACM
		Transactions on Graphics (TOG). vol.~30, p.~79. ACM (2011)
		
		\bibitem{korman2014probably}
		Korman, S., Litman, R., Avidan, S., Bronstein, A.M.: Probably approximately
		symmetric: Fast 3d symmetry detection with global guarantees. CoRR abs
		\textbf{1403}, ~2 (2014)
		
		\bibitem{kovnatsky2013coupled}
		Kovnatsky, A., Bronstein, M.M., Bronstein, A.M., Glashoff, K., Kimmel, R.:
		Coupled quasi-harmonic bases. In: Computer Graphics Forum. vol.~32, pp.
		439--448. Wiley Online Library (2013)
		
		\bibitem{kurz2014symmetry}
		Kurz, C., Wu, X., Wand, M., Thorm{\"a}hlen, T., Kohli, P., Seidel, H.P.:
		Symmetry-aware template deformation and fitting. In: Computer Graphics Forum.
		vol.~33, pp. 205--219. Wiley Online Library (2014)
		
		\bibitem{li2016efficient}
		Li, B., Johan, H., Ye, Y., Lu, Y.: Efficient 3d reflection symmetry detection:
		A view-based approach. Graphical Models  \textbf{83},  2--14 (2016)
		
		\bibitem{li2015approximate}
		Li, C., Wand, M., Wu, X., Seidel, H.P.: Approximate 3d partial symmetry
		detection using co-occurrence analysis. In: 3D Vision (3DV), 2015
		International Conference on. pp. 425--433. IEEE (2015)
		
		\bibitem{lipman2010symmetry}
		Lipman, Y., Chen, X., Daubechies, I., Funkhouser, T.: Symmetry factored
		embedding and distance. In: ACM Transactions on Graphics (TOG). vol.~29,
		p.~103. ACM (2010)
		
		\bibitem{liu2015properly}
		Liu, X., Li, S., Liu, R., Wang, J., Wang, H., Cao, J.: Properly constrained
		orthonormal functional maps for intrinsic symmetries. Computers \& Graphics
		\textbf{46},  198--208 (2015)
		
		\bibitem{liu2010computational}
		Liu, Y., Hel-Or, H., Kaplan, C.S., Van~Gool, L., et~al.: Computational symmetry
		in computer vision and computer graphics. Foundations and
		Trends{\textregistered} in Computer Graphics and Vision  \textbf{5}(1--2),
		1--195 (2010)
		
		\bibitem{loy2006detecting}
		Loy, G., Eklundh, J.O.: Detecting symmetry and symmetric constellations of
		features. In: European Conference on Computer Vision. pp. 508--521. Springer
		(2006)
		
		\bibitem{lukavc2017nautilus}
		Luk{\'a}{\v{c}}, M., S{\`y}kora, D., Sunkavalli, K., Shechtman, E.,
		Jamri{\v{s}}ka, O., Carr, N., Pajdla, T.: Nautilus: recovering regional
		symmetry transformations for image editing. ACM Transactions on Graphics
		(TOG)  \textbf{36}(4), ~108 (2017)
		
		\bibitem{martinet2006accurate}
		Martinet, A., Soler, C., Holzschuch, N., Sillion, F.X.: Accurate detection of
		symmetries in 3d shapes. ACM Transactions on Graphics (TOG)  \textbf{25}(2),
		439--464 (2006)
		
		\bibitem{mitra2006partial}
		Mitra, N.J., Guibas, L.J., Pauly, M.: Partial and approximate symmetry
		detection for 3d geometry. ACM Transactions on Graphics (TOG)
		\textbf{25}(3),  560--568 (2006)
		
		\bibitem{mitra2007symmetrization}
		Mitra, N.J., Guibas, L.J., Pauly, M.: Symmetrization. In: ACM Transactions on
		Graphics (TOG). vol.~26, p.~63. ACM (2007)
		
		\bibitem{mitra2008symmetry}
		Mitra, N.J., Pauly, M.: Symmetry for architectural design. Advances in
		Architectural Geometry pp. 13--16 (2008)
		
		\bibitem{mitra2013symmetry}
		Mitra, N.J., Pauly, M., Wand, M., Ceylan, D.: Symmetry in 3d geometry:
		Extraction and applications. In: Computer Graphics Forum. vol.~32, pp. 1--23.
		Wiley Online Library (2013)
		
		\bibitem{mitra2014structure}
		Mitra, N.J., Wand, M., Zhang, H., Cohen-Or, D., Kim, V., Huang, Q.X.:
		Structure-aware shape processing. In: ACM SIGGRAPH 2014 Courses. p.~13. ACM
		(2014)
		
		\bibitem{mitra2010illustrating}
		Mitra, N.J., Yang, Y., Yan, D., Li, W., Agrawala, M.: Illustrating how
		mechanical assemblies work. ACM Transactions on Graphics  (2010)
		
		\bibitem{o1983semi}
		O'neill, B.: Semi-Riemannian geometry with applications to relativity,
		vol.~103. Academic press (1983)
		
		\bibitem{ovsjanikov2012functional}
		Ovsjanikov, M., Ben-Chen, M., Solomon, J., Butscher, A., Guibas, L.: Functional
		maps: a flexible representation of maps between shapes. ACM Transactions on
		Graphics (TOG)  \textbf{31}(4), ~30 (2012)
		
		\bibitem{ovsjanikov2008global}
		Ovsjanikov, M., Sun, J., Guibas, L.: Global intrinsic symmetries of shapes. In:
		Computer graphics forum. vol.~27, pp. 1341--1348. Wiley Online Library (2008)
		
		\bibitem{panozzo2012fields}
		Panozzo, D., Lipman, Y., Puppo, E., Zorin, D.: Fields on symmetric surfaces.
		ACM Transactions on Graphics (TOG)  \textbf{31}(4), ~111 (2012)
		
		\bibitem{pinkall1993computing}
		Pinkall, U., Polthier, K.: Computing discrete minimal surfaces and their
		conjugates. Experimental mathematics  \textbf{2}(1),  15--36 (1993)
		
		\bibitem{podolak2006planar}
		Podolak, J., Shilane, P., Golovinskiy, A., Rusinkiewicz, S., Funkhouser, T.: A
		planar-reflective symmetry transform for 3d shapes. ACM Transactions on
		Graphics (TOG)  \textbf{25}(3),  549--559 (2006)
		
		\bibitem{raviv2010full}
		Raviv, D., Bronstein, A.M., Bronstein, M.M., Kimmel, R.: Full and partial
		symmetries of non-rigid shapes. International journal of computer vision
		\textbf{89}(1),  18--39 (2010)
		
		\bibitem{raviv2010diffusion}
		Raviv, D., Bronstein, A.M., Bronstein, M.M., Kimmel, R., Sapiro, G.: Diffusion
		symmetries of non-rigid shapes. In: Proc. 3DPVT. vol.~2. Citeseer (2010)
		
		\bibitem{reuter2009discrete}
		Reuter, M., Biasotti, S., Giorgi, D., Patan{\`e}, G., Spagnuolo, M.: Discrete
		laplace--beltrami operators for shape analysis and segmentation. Computers \&
		Graphics  \textbf{33}(3),  381--390 (2009)
		
		\bibitem{shehu2014characterization}
		Shehu, A., Brunton, A., Wuhrer, S., Wand, M.: Characterization of partial
		intrinsic symmetries. In: European Conference on Computer Vision. pp.
		267--282. Springer (2014)
		
		\bibitem{shi2016symmetry}
		Shi, Z., Alliez, P., Desbrun, M., Bao, H., Huang, J.: Symmetry and orbit
		detection via lie-algebra voting. In: Computer Graphics Forum. vol.~35, pp.
		217--227. Wiley Online Library (2016)
		
		\bibitem{sipiran2014approximate}
		Sipiran, I., Gregor, R., Schreck, T.: Approximate symmetry detection in partial
		3d meshes. In: Computer Graphics Forum. vol.~33, pp. 131--140. Wiley Online
		Library (2014)
		
		\bibitem{speciale2016symmetry}
		Speciale, P., Oswald, M.R., Cohen, A., Pollefeys, M.: A symmetry prior for
		convex variational 3d reconstruction. In: European Conference on Computer
		Vision. pp. 313--328. Springer (2016)
		
		\bibitem{sun2009concise}
		Sun, J., Ovsjanikov, M., Guibas, L.: A concise and provably informative
		multi-scale signature based on heat diffusion. In: Computer graphics forum.
		vol.~28, pp. 1383--1392. Wiley Online Library (2009)
		
		\bibitem{sung2015data}
		Sung, M., Kim, V.G., Angst, R., Guibas, L.: Data-driven structural priors for
		shape completion. ACM Transactions on Graphics (TOG)  \textbf{34}(6), ~175
		(2015)
		
		\bibitem{surazhsky2005fast}
		Surazhsky, V., Surazhsky, T., Kirsanov, D., Gortler, S.J., Hoppe, H.: Fast
		exact and approximate geodesics on meshes. In: ACM transactions on graphics
		(TOG). vol.~24, pp. 553--560. Acm (2005)
		
		\bibitem{thomas2013detecting}
		Thomas, D.M., Natarajan, V.: Detecting symmetry in scalar fields using
		augmented extremum graphs. IEEE transactions on visualization and computer
		graphics  \textbf{19}(12),  2663--2672 (2013)
		
		\bibitem{wang2017group}
		Wang, H., Huang, H.: Group representation of global intrinsic symmetries. In:
		Computer Graphics Forum. vol.~36, pp. 51--61. Wiley Online Library (2017)
		
		\bibitem{wang2011symmetry}
		Wang, Y., Xu, K., Li, J., Zhang, H., Shamir, A., Liu, L., Cheng, Z., Xiong, Y.:
		Symmetry hierarchy of man-made objects. In: Computer Graphics Forum. vol.~30,
		pp. 287--296. Wiley Online Library (2011)
		
		\bibitem{wu2014real}
		Wu, X., Wand, M., Hildebrandt, K., Kohli, P., Seidel, H.P.: Real-time
		symmetry-preserving deformation. In: Computer Graphics Forum. vol.~33, pp.
		229--238. Wiley Online Library (2014)
		
		\bibitem{xiao2015content}
		Xiao, C., Jin, L., Nie, Y., Wang, R., Sun, H., Ma, K.L.: Content-aware model
		resizing with symmetry-preservation. The Visual Computer  \textbf{31}(2),
		155--167 (2015)
		
		\bibitem{xu2012multi}
		Xu, K., Zhang, H., Jiang, W., Dyer, R., Cheng, Z., Liu, L., Chen, B.:
		Multi-scale partial intrinsic symmetry detection. ACM Transactions on
		Graphics (TOG)  \textbf{31}(6), ~181 (2012)
		
		\bibitem{xu2009partial}
		Xu, K., Zhang, H., Tagliasacchi, A., Liu, L., Li, G., Meng, M., Xiong, Y.:
		Partial intrinsic reflectional symmetry of 3d shapes. In: ACM Transactions on
		Graphics (TOG). vol.~28, p.~138. ACM (2009)
		
		\bibitem{yoshiyasu2016symmetry}
		Yoshiyasu, Y., Yoshida, E., Guibas, L.: Symmetry aware embedding for shape
		correspondence. Computers \& Graphics  \textbf{60},  9--22 (2016)
		
		\bibitem{zelditch2013eigenfunctions}
		Zelditch, S.: Eigenfunctions and nodal sets. Surveys in Differential Geometry
		\textbf{18}(1),  237--308 (2013)
		
		\bibitem{zheng2015skeleton}
		Zheng, Q., Hao, Z., Huang, H., Xu, K., Zhang, H., Cohen-Or, D., Chen, B.:
		Skeleton-intrinsic symmetrization of shapes. In: Computer Graphics Forum.
		vol.~34, pp. 275--286. Wiley Online Library (2015)
		
		\bibitem{petersen2008matrix}
		Petersen, K.B., Pedersen, M.S., et~al.: The matrix cookbook. Technical
		University of Denmark  \textbf{7}(15), ~510 (2008)
		
		
	\end{thebibliography}
	
\end{document}